\documentclass{article}
\pdfoutput=1 
\usepackage[accepted]{faohack}
\usepackage[accepted]{faohack}

\usepackage[utf8]{inputenc} 
\usepackage[T1]{fontenc}    
\usepackage{url}            
\usepackage{booktabs}       
\usepackage{microtype}      
\usepackage{natbib}
\setcitestyle{authoryear,round,citesep={;},aysep={,},yysep={;}}

\usepackage{graphicx}
\usepackage{subfig}
\usepackage{array}
\graphicspath{{.}
              {figs/}
              {fig1/figs/}
              {fig1b/figs/}
              {fig1b/}
              {fig2/figs/}
              {fig_low_lr/}
              {gan_no_bn_appendix/figs/}
              {sub_main-comparison/}
              {fig_interleaved_training/}
             }

\usepackage{algpseudocode}

\usepackage{amsfonts}       
\usepackage{amsmath}
\usepackage{amssymb}

\usepackage{verbatim}
\usepackage{float}
\usepackage{amsthm}

\makeatletter

\pdfpageheight\paperheight
\pdfpagewidth\paperwidth

\floatstyle{ruled}
\newfloat{algorithm}{tbp}{loa}
\providecommand{\algorithmname}{Algorithm}
\floatname{algorithm}{\protect\algorithmname}

\providecommand{\definitionname}{Definition}
\providecommand{\factname}{Fact}
\providecommand{\theoremname}{Theorem}
\theoremstyle{plain}
\newtheorem{fact}{\protect\factname}
\newtheorem{thm}{\protect\theoremname}
\newtheorem{prop}{Proposition}
\newtheorem{cor}{Corollary}

\theoremstyle{definition}
\newtheorem{defn}{\protect\definitionname}

\providecommand{\defas}     {\mathrel{\triangleq}}

\providecommand{\funcName}[1]{\textsc{#1}}

\newcommand{\E}{\mathbf{E}}

\providecommand{\reals}{\mathbb {R} }

\newcommand{\set}[1]{\mathcal{#1}}

\providecommand{\indexMarkup}[1]{(#1)}

 \providecommand{\argsA}[2]{ {#1}_{#2} } 
 \providecommand{\argsAI}[3]{ {#1}_{#2}^{\indexMarkup{#3}} }

\providecommand{\sect}{Section~}

\providecommand{\fig}{Figure~}

\providecommand{\rthm}{Theorem~}

\providecommand{\AGS}{D}
\providecommand{\agentS}{\set{\AGS}}
\providecommand{\agentI}[1]{{#1}}
\providecommand{\nrA}{n} 

\providecommand{\AC}{s}                 
\providecommand{\ACS}{\set{S}}                     

\providecommand{\aA}[1]{\argsA{\AC}{#1}}
\providecommand{\aAS}[1]{\argsA{\ACS}{#1}}      
\providecommand{\aAI}[2]{\argsAI{\AC}{#1}{#2}}

\providecommand{\utF}{u}		
\providecommand{\utA}[1]{\argsA{\utF}{#1}}	

\usepackage{verbatim}

\usepackage{enumitem}
\setlist{topsep=2mm,partopsep=0mm,parsep=1mm,itemsep=.1ex,leftmargin=5mm}

\title{Beyond Local Nash Equilibria \\ for Adversarial Networks\footnote{%
This version supersedes our earlier arXiv paper~\citep{arXiv1}.}}
\author{
    Frans A. Oliehoek \and
    Rahul Savani \and
    Jose Gallego \and
    Elise van der Pol \and
    Roderich Gro{\ss}}
\begin{document}

\maketitle

\begin{abstract}
Save for some special cases, current training methods for 
Generative Adversarial Networks (GANs) are at best guaranteed to 
converge to a ‘local Nash equilibrium’ (LNE).
Such LNEs, however, can be arbitrarily far from an actual Nash equilibrium (NE), which implies that there are no guarantees on the quality of the found generator or classifier.
This paper proposes to  model GANs explicitly as finite games in mixed
strategies, thereby ensuring that every LNE is an NE. With this formulation, we
propose a solution method that is proven to monotonically converge to a
\emph{resource-bounded Nash equilibrium (RB-NE)}: by increasing computational
resources we can find better solutions. We empirically demonstrate that our
method is less prone to typical GAN problems such as mode collapse, and produces
solutions that are less exploitable than those produced by GANs and MGANs, and 
closely resemble theoretical predictions about NEs. 
\end{abstract}

\thispagestyle{empty}



\section*{}

\global\long\def\nrA{n}

\global\long\def\aA#1{\argsA{\AC}{#1}}

\global\long\def\utA#1{\argsA{\utF}{#1}}

\global\long\def\mA#1{\argsA{\mu}{#1}}

\global\long\def\aAI#1#2{\argsAI{\AC}{#1}{#2}}

\global\long\def\aAS#1{\argsA{\ACS}{#1}}

\global\long\def\gameval{v^{*}}

\global\long\def\mf{\phi}
 %

\global\long\def\paramG{\theta_{G}}
 %

\global\long\def\paramC{\theta_{C}}
 %

\global\long\def\paramGS{\Theta_{G}}
 %

\global\long\def\paramCS{\Theta_{C}}
 %

\global\long\def\dim{d}
 %


\section{Introduction}

Generative Adversarial Networks (GANs) \citep{Goodfellow14NIPS27}
are a framework in which two neural networks compete with each other:
the \emph{generator (G) }tries to trick the \emph{classifier (C)}
into classifying its generated fake data as true. GANs hold great
promise for the development of accurate generative models for
complex distributions. 
Consequently, in just a few years, GANs have grown
into a major topic of research in machine learning. A core appeal
is that they do not need to rely on distance metrics~\citep{Li16SI}.
However, GANs are difficult to train 
\citep{Unterthiner18ICLR,ArjovskyB17,ACB17_WGAN}. A typical
problem is \emph{mode collapse}, which can take the form of 
\emph{mode omission}, where the generator does not produce 
points from certain modes, or \emph{mode degeneration}, in which
for at least one mode the generator only partially covers the mode.
Moreover, while learning the players
may \emph{forget}: e.g., it is possible that a
classifier correctly learns to classify part of the input space as
`fake' only to forget this later when the generator no longer generates
examples in this part of space. 
In fact, except for very special cases (cf. \sect\ref{sec:relatedWork}), the
best current training methods can offer
\citep{Heusel17NIPS,Unterthiner18ICLR} is a guarantee to converge to a  \emph{local} Nash
equilibrium (LNE)~\citep{Ratliff13ACCCC}. However, an LNE can be arbitrarily far from an NE 
(which we will also refer to as `global NE' to discriminate)
and the corresponding generator might be exploitable by a strong opponent due
to suffering from problems such as mode collapse.
Moreover, adding computational resources alone may not offer a way to escape
these local equilibria: the problem does not lie in the lack of computational
resources, but is inherently the result of only allowing small steps in
strategy space using gradient-based training.

We introduce a novel approach that does not suffer from getting trapped in LNEs:
finite \emph{Generative Adversarial Network Games (GANGs)} formulate adversarial networks
as \emph{finite} zero-sum games, and the solutions that
we try to find are saddle points in \emph{mixed strategies}. This approach is
motivated by the observation that, considering a GAN as a finite zero-sum game,
in the space of mixed strategies, any local Nash equilibrium is a global one.
Intuitively, the reason for this is that whenever there is a profitable pure
strategy deviation one can move towards it in the space of mixed strategies.

Since we cannot expect to find exact best responses due to the extremely large
number of pure strategies that result for sensible choices of neural network
classes, we introduce resource-bounded best-responses (RBBRs), and the
corresponding resource-bounded Nash equilibrium (RB-NE), which is a pair of
mixed strategies in which no player can find a better RBBR. This is richer than
the notion of local Nash equilibrium in that it captures not only failures of
escaping local optima of gradient descent, but applies to \emph{any} approximate
best response computations, including methods with random restarts, and allows
us to provide convergence guarantees. 

The key features of our approach are that:
\begin{itemize}[leftmargin=3mm]
\item It is based on 
        finite zero-sum games, and as such it
        enables the use of existing game-theoretic methods. In this paper we 
		focus on one such method, Parallel Nash Memory (PNM)~\citep{Oliehoek06GECCO}.
\item It will not get trapped in LNEs: we prove that it monotonically converges to an
    RB-NE, which means that more computation can improve solution quality.
\item Moreover, it works for any network architecture (unlike previous
        approaches, see  \sect\ref{sec:relatedWork}). In particular, future improvements in
        classifiers/generator networks can be exploited directly.
\end{itemize}



We investigate empirically the effectiveness of PNM and show that it can indeed
deal well with typical GAN problems such as mode collapse and forgetting, 
especially in distributions with asymmetric structure to
their modes. 
%
We show that the found solutions are much less susceptible to being exploited by
an adversary, and that they more
closely match the theoretical predictions made by \citet{Goodfellow14NIPS27}
about the conditions at a Nash equilibrium.


\section{Background }

\label{sec:background}  
We defer a more detailed treatment of related work on GANs and recent game-theoretic approaches until Section~\ref{sec:relatedWork}.
Here, we introduce some basic game-theoretic notation.
\begin{defn}[`game']
    A \emph{two-player strategic game}
, which we will
    simply call `game', 
    is a tuple $\left\langle 
    \agentS,
    \left\{ \aAS i\right\} _{i\in\agentS},
    \left\{ \utA i\right\} _{i\in\agentS}
    \right\rangle $,
where 
    $\agentS=\{\agentI1,\agentI2\}$ 
    is the set of \emph{players}, 
	$\aAS i$ is the set of \emph{pure strategies} (actions) for player $i$, 
    and $\utA i:\mathcal{S}\to\mathbb{R}$ is $i's$ payoff function
    defined on the set of pure strategy profiles 
    $\mathcal{{S}}:=\aAS 1\times\aAS 2$. 
    When the 
    action sets are finite, the game is \emph{finite}.
\end{defn}
We also write $\aA i$ and $\aA{-i}$ for the strategy  of agent $i$ and its
opponent respectively. 

A fundamental concept is the Nash equilibrium (NE), which
is a strategy profile such that no player can unilaterally deviate
and improve his payoff. 
\begin{defn}[Pure Nash equilibrium]
 A pure strategy profile 
    $\aA{} = \left\langle \aA{i},\aA{-i} \right\rangle $
is an NE if and only if
    $\utA i( \aA{} ) \geq \utA i(\left\langle \aA i', \aA{-i}\right\rangle )$
for all players $i$ and $\aA i'\in\aAS i$.
\end{defn}
A finite game may not possess a pure NE. A \emph{mixed strategy} $\mA i$ of
player $i$ is a probability distribution over $i$'s pure strategies
$\aAS i$. The set of such probability distributions is denoted by $\Delta(\aAS i)$. The payoff of a player under a profile of mixed strategies
$\mA{} = \left\langle \mA{1},\mA{2}\right\rangle $ 
is defined as the expectation: 
$
    \utA i(\mA{}):=
                \sum_{\aA{}\in\aAS{}} 
                [ \prod_{j\in \agentS}\mA{j}(\aA{j}) ] \cdot \utA{i}(\aA{}).
$                

Then an NE in mixed strategies is defined as follows. 
\begin{defn}[Mixed Nash equilibrium]
    A 
    $\mA{} = \left\langle \mA i, \mA{-i} \right\rangle $
    is an NE if and only if
    $\utA i(\mA{})
     \geq
     \utA i(\left\langle \aA i',\mA{-i}\right\rangle )$
for all players $i$ and potential unilateral deviations $\aA i'\in\aAS i$.
\end{defn}
Every finite game has at least one NE in mixed strategies~\citep{Nash50}.
In this paper we deal with two-player 
\emph{zero-sum} games, where
$\utA 1(\aA 1,\aA 2)=-\utA 2(\aA 1,\aA 2)$ for all $\aA 1\in\aAS 1,\aA 2\in\aAS 2$.
The equilibria of zero-sum games, also called \emph{saddle points},\footnote{
    \label{foot:saddlepoint}
    Note that in game theory the term `saddle point' is used to denote a `global' saddle point which corresponds to a Nash equilibrium: there is no profitable deviation near or far away from the current point. In contrast, in machine learning, the term 'saddle point' typically denotes a `local' saddle point: no player can improve its payoff by making a small step from the current joint strategy.}
have several important properties,
as stated in the Minmax theorem.
\begin{thm}[\citet{vNM28}]
\label{thm:minmax}In a finite zero-sum game, 
$
    \min_{\mA{2}}
    \max_{\mA{1}}
    u_{1}(\mA{})
    =
    \max_{\mA{1}}
    \min_{\mA{2}}
    u_{1}(\mA{})
    =\gameval,
$
where $\gameval$ is  the \emph{value} of the game. 
\end{thm}

All equilibria have payoff $\gameval$
and equilibrium strategies are interchangeable:
if
$ \langle \mA{1}, \mA{2} \rangle $
and
$ \langle \mA{1}',\mA{2}' \rangle $
are equilibria, then so are
$ \langle \mA{1}',\mA{2} \rangle $
and
$ \langle \mA{1} ,\mA{2}' \rangle $ \citep{Osborne+Rubinstein94}.
This means that in zero-sum games we do not need to worry about equilibrium selection: any equilibrium strategy for a player is guaranteed to achieve the value of the game.
Moreover, the convex combination of two equilibria is an equilibrium, 
meaning that  the game has either one or infinitely many equilibria.

We also
employ
the standard, additive notion of \emph{approximate equilibrium:}
\begin{defn}
    A pair of (possibly pure) strategies $(\mA i,\mA{-i})$ is an $\epsilon$-NE if
    $
    \forall i \quad
    \utA i(\mA i,\mA{-i})
    \geq
    \max_{\mA i'}
    \utA i(\mA i',\mA{-i})-\epsilon.
    $
In other words, no player can gain more than $\epsilon$ by deviating.
\end{defn}

In the literature, GANs have not typically been considered as finite
games. The natural interpretation of the standard setup of GANs is
of an infinite game where payoffs are defined over all possible weight
parameters for the respective neural networks. With this view we do
not obtain existence of saddle points in the space of parameters%
\footnote{
    Note that \citet{Goodfellow14NIPS27}'s theoretical results on the existance
    of an equilibrium are ``done in a non-parametric setting'' which 
    studies ``convergence in the space of
    probability density functions'', not parameters.
}, 
nor the desirable properties of Theorem \ref{thm:minmax}.
Some results on the existence of saddle
points in infinite action games are known, but they require properties
like convexity and concavity of utility functions \citep{Aubin98optima}, which we cannot
apply as they would need to hold w.r.t. the neural network parameters.
This is why the notion of  \emph{local Nash equilibrium (LNE)} has
arisen in the literature \citep{Ratliff13ACCCC,Unterthiner18ICLR}.
Roughly, an LNE is a strategy profile where neither player can improve
in a small neighborhood of the profile.\footnote{So this corresponds to a `local saddle point'. Cf.~footnote~\ref{foot:saddlepoint}.
}
In finite games
every LNE is an NE, as, whenever there is a global deviation (i.e., a better response), one
can always deviate locally in the space of mixed strategies 
towards a pure best response (by playing that better response with $\epsilon$ higher probability).


\section{GANGs}

\label{sec:GANGs}

In order to capitalize on the insight that we can escape local equilibria by
switching to mixed strategy space for a finite game, we formalize adversarial
networks in a (finite) games setting, that we call (finite) \emph{Generative
Adversarial Network Games (GANGs)}.\footnote{
    Since all GANs are in practice implemented in finite precision floating point
    systems, we limit ourselves to the mathematically simpler setting of finite games.
    However, we point out that at least one convergent method 
    (fictitious play)
    for continuous zero-sum games exists
    \citep{Danskin1981}. We conjecture that it is possible to show such convergence also
    for the Parallel Nash Memory method that we will employ in this paper.
}



We start with making explicit how general (infinite) GANs correspond to strategic games, via the GANG framework:

\begin{defn}[GANG]
\label{def:A-GANG-is}
A \emph{GANG} is a tuple 
$\mathcal{M} = \left\langle p_{d},\left\langle G,p_{z}\right\rangle ,C,\mf\right\rangle $
with 
\begin{itemize}[leftmargin=3mm]
\item $p_{d}(x)$ is the distribution over (`true' or `real') data points
$x\in\mathbb{R}^{\dim}$.
\item $G$ is a neural network class 
    parametrized by a
    parameter vector $\paramG\in\paramGS$
and $d$ outputs, such that $G(z;\paramG)\in\mathbb{R}^{\dim}$ denotes
the (`fake' or `generated') output of $G$ on a random vector $z$
drawn from some distribution $z\sim p_{z}$. 
\item $C$ is a neural network class 
    parametrized by a
    parameter vector $\paramC\in\paramCS$
and a single output, such that the output $C(x;\paramC)\in[0,1]$
indicates the `realness' of $x$ according to $C$. 
\item ${\mf:[0,1]\to\reals}$ is a \emph{measuring function} \citep{Arora17ICML}---e.g., $\log$ for GANs, the identity mapping for WGANs---%
used to specify game payoffs, explained next.
\end{itemize}
\end{defn}
A GANG induces a zero-sum game in an intuitive way:
\begin{defn}
\label{def:The-induced-zero-sum}The induced zero-sum strategic-form
game of a GANG is $\left\langle \agentS=\{G,C\},\left\{ \aAS G,\aAS C\right\} ,\left\{ \utA G,\utA C\right\} \right\rangle $
with:
\begin{itemize}[leftmargin=3mm]
\item 
    $\aAS G=\left\{ G(\cdot;\paramG)\mid\paramG\in\paramGS\right\} $
    the set of strategies
    $\aA G$;
\item 
    $\aAS C=\left\{ C(\cdot;\paramC)\mid\paramC\in\paramCS\right\} $,
    the set of strategies
    $\aA C$;
\item 
    $\utA C(\aA G,\aA C) =
        \E_{x\sim p_{d}}
        [\mf(\aA C(x))]
        -
        \E_{z\sim p_{z}}
        [\phi(\aA C(\aA G(z)))].$
I.e., the score of $C$ is the expected `measured realness'
of the real data minus that of the fake data;
\item 
    $\utA G(\aA G,\aA C)=-\utA C(\aA G,\aA C)$.
\end{itemize}
\end{defn}
As such, when using $\mf = \log $, GANGs employ a payoff function for $G$ that
use \citet{Goodfellow14NIPS27}'s trick to enforce strong gradients early in the
training process, but it applies this transformation to  $\utA C$ too, in order
to retain the zero-sum property. The correctness of these transformations are
proven in Appendix~\ref{app:zerosum}.

In practice, GANs are represented using floating point numbers, of which, for 
a given setup, there 
is only a finite (albeit large) number. In GANGs, we formalize this:
\begin{defn}
Any GANG where $G,C$ are finite classes---i.e.,
classes of networks constructed from a finite set of node types (e.g.,
\{Sigmoid, ReLu, Linear\})---and with
architectures of bounded size, is called a \emph{finite network
class GANG.} 
A finite network class GANG in which the sets 
$\paramGS,\paramCS$
are finite too, it is called a \emph{finite GANG. }
\end{defn}
From now on, we will focus on finite GANGs.
We emphasize this finiteness, because this is exactly what enables us to obtain
the desirable properties mentioned in \sect\ref{sec:background}:
existence of (one or infinitely many) mixed NEs with the same value, as well
as the guarantee that any LNE is an NE.
Moreover, these properties hold for 
the GANG in its original formulation---%
not for a theoretical abstraction in terms of (infinite capacity) densities---%
which means that we can truly expect solution methods (that operate in the
parametric domain) to exploit these properties.
However, since we do not impose any additional constraints or
discretization%
\footnote{
Therefore, our finite GANGs have the same representational
capacity as normal GANs that are implemented using floating point arithmetic.
When strategies would be highly non-continuous functions of  $\paramG$ and
$\paramC$, then a discretization (e.g., a floating point system) might be
inadequate, but then GANs are not appropriate either:  neural networks are continuous
functions.
}%
, the number of strategies (all possible unique instantiations of
the network class with floating point numbers) is \emph{huge}. 
Therefore, we think that finding (near-) equilibria with
small supports is one of the most important challenges for making
principled advances in  the field of adversarial networks.
As a first step towards addressing this challenge, we propose to make use of the
\emph{Parallel Nash Memory (PNM)}~\citep{Oliehoek06GECCO}, which can be seen as a generalization 
(to non-exact best responses) of the \emph{double oracle method}~\citep{McMahan03ICML}.
%


\section{Resource-Bounded GANGs}

\label{sec:RB-GANGs}

While finite GANGs do not suffer from LNEs, solving them is non-trivial: even
though we know that there always is a direction (i.e., a best response) that 
we can play more frequently to move towards an NE, computing that direction
itself is a computationally intractable task. There is huge number of candidate
strategies and the best-response payoff is a non-linear (or convex) function of
the chosen strategy.
This means that finding an NE, or even an $\epsilon-$NE
will be typically beyond our capabilities.
Therefore we consider players with bounded computational resources. 

\paragraph{Resource-Bounded Best-Responses (RBBR).} 
A Nash equilibrium is defined by the absence of better responses for any of the
players. As such, best response computation is a critical tool to verify whether
a strategy profile is an NE, and is also a common subroutine in algorithms that
compute an NE.%
\footnote{
    We 
    motivate the need for an alternative solution concept.
    We acknowledge the fact that there are solution methods 
    that do not require the explicit computation for best responses, but are not aware
    of any such methods that would be scalable to the huge strategy spaces
    of finite GANGs.
}
However, 
since  
computing an
($\epsilon$-)best response will generally be intractable for GANGs, we examine
the type of solutions that we actually can expect to
compute with bounded computational power by the notion of
\emph{resource-bounded best response} and show how it naturally leads to 
the \emph{resource-bounded NE (RB-NE)} solution concept.

We say that $\aAS i^{RB}\subseteq\aAS i$ is the subset of strategies
of player $i$, that $i$ can compute as a best response, given
its bounded computational resources.
This computable set is an abstract formulation that can capture 
gradient descent being trapped in local optima, 
as well as other 
phenomena that prevent us from computing a best response 
(e.g., not even reaching a local optimum in the available time). 
\begin{defn}
    \label{def:RBBR}
A strategy $s_{i}\in\aAS i^{RB}$ of player~$i$ is a \emph{resource-bounded
best-response} (RBBR) against a (possibly mixed) strategy  $\mA j$, if $\forall\aA i'\in\aAS i^{RB},~~\utA i(\aA i,\mA j)\geq\utA i(\aA i',\mA j)$.
\end{defn}
That is, $\aA i$ only needs to be amongst the best strategies
that player~$i$ \emph{can compute} in response to $\mA j$. We denote
the set of such RBBRs to $\mA j$ by
$\aAS i^{RBBR(\mA j)}\subseteq\aAS i^{RB}$. 
\begin{defn}
A \emph{resource-bounded best-response function}\textbf{ $f_{i}^{RBBR}:\Delta(\aAS j)\to\aAS i^{RB}$
}is a function that maps from the set of possible strategies of player $j$ to an RBBR for~$i$, s.t.\ $\forall_{\mA j}\; f_{i}^{RBBR}(\mA j)\in\aAS i^{RBBR(\mA j)}$.
\end{defn}

Using RBBRs, we  define an intuitive specialization of NE:

\begin{defn}
    $\mA{}=\langle \mA i,\mA j \rangle $ is a \emph{resource-bounded 
    NE (RB-NE)} 
iff 
$\forall i\quad\utA i(\mA i,\mA j)\geq\utA i(f_{i}^{RBBR}(\mA j),\mA j).$
\end{defn}
That is, an RB-NE can be thought of as follows:
we present $\mA{}$ to each player $i$ and it gets the chance to
switch to another strategy, for which it can apply its bounded resources 
(i.e., use $f_{i}^{RBBR}$) exactly once. After this application,
the player's resources are exhausted and if the found $f_{i}^{RBBR}(\mA
j)$ does not lead to a higher payoff it will not have an incentive to
deviate.\footnote{
Of course, during training the RBBR functions will be used many times, and
therefore the overall computations that \emph{we, the designers (`GANG
trainers'),} use is a multitude of the computational capabilities that that the
RB-NE assumes. However, this is of no concern: the goal of the RB-NE is to provide a characterization of the end
point of training. In particular, it allows us to give an abstract
formalization of the idea that more computational resources imply a better
solution.}

Clearly, an RB-NE can be linked to the familiar notion of $\epsilon$-NE
by making assumptions on the power of the best response computation.

\begin{thm}
\label{thm:epsNashGivenSufficientResources}
If both players are powerful enough to compute $\epsilon$-best
responses, then an RB-NE is an $\epsilon$-NE.
\begin{proof}
    \sloppypar
Starting from the RB-NE $(\mA i,\mA j)$, assume an arbitrary~$i$.
By definition of RB-NE $\utA i(\mA i,\mA j)\geq \utA i(f_{i}^{RBBR}(\mA j),\mA j)\geq\max_{\mA i'}\utA i(\mA i',\mA j)-\epsilon$.\end{proof}
\end{thm}

\paragraph{Non-deterministic Best Responses. }The above definitions 
assumed deterministic 
RBBR functions
\textbf{$f_{i}^{RBBR}$}. However, in many cases the RBBR function can be
non-deterministic (e.g., due to random restarts),
which means that the sets $\aAS i^{RB}$ are non-deterministic.
This is not a fundamental problem, however, and the same approach can
be adapted to allow for such non-determinism. 

In particular, now let
$f_{i}^{RBBR}$ be a non-deterministic function, and \emph{define}
$\aAS i^{RB}$ as the 
range of this function.
That is, we define $\aAS i^{RB}$ as that set of strategies that
our non-deterministic RBBR function might deliver. Given this modification
the definition of the RB-NE remains unchanged:
a strategy profile $\mA{}=\langle \mA i,\mA j \rangle $ is a
\emph{non-deterministic RB-NE} 
if 
each player $i$ uses all its computational resources by calling 
$f_{i}^{RBBR}(\mA j)$ once, 
and no player finds a better strategy to switch to.

\section{Solving GANGs}

\label{sec:SolvingGANGs}

Treating GANGs as finite games in mixed strategies 
permits
building on existing tools and algorithms
for these classes of games \citep{Fudenberg98book,RakhlinS13a,Foster16NIPS}.
In this section, we describe how to use the Parallel Nash Memory \citep{Oliehoek06GECCO},
which 
explicitly aims to prevent forgetting,
is particularly tailored to finding approximate NEs with small support, 
and monotonically%
\footnote{
    For an explanation of the precise meaning of monotonic here, we refer to \citet{Oliehoek06GECCO}.
    Roughly, we will be `secure' against more strategies of the other agent with each iteration.
    This does not imply that the worst case payoff for an agent also improves monotonicaly. 
    The latter property, while desirable, is not possible
    with an approach that incrementally constructs sub-games of the full game, as considered here: there might always be a part of the game we have not seen yet, but which we might discover in the future that will lead to a very poor worst case payoff for one of the agents.
}
converges to such an equilibrium.

In the following, we give a concise description of a slightly simplified form of
Parallel Nash Memory (PNM) and how we apply it to GANGs. %
For ease of explanation, we focus on the setting with deterministic 
best responses.\footnote{
    In our experiments, we use random initializations for the best responses.
    To deal with this non-determinism we simply discard any tests that are not able
    to find a positive payoff over the current mixed strategy NE $\left\langle \mA G,\mA C\right\rangle $,
    but we do not terminate. Instead, we run
    for a pre-specified number of iterations. 
}
 
\begin{algorithm}
\providecommand{\commentSymb}{//}
\begin{algorithmic}[1]
\small
\State{$\langle \aA{G}, \aA{C} \rangle \gets \funcName{InitialStrategies}()                                                  $}
\State{$\langle \mA{G}, \mA{C} \rangle \gets \langle \{\aA{G}\}, \{\aA{C}\} \rangle    $} \Comment{set initial mixtures}
\While{True}
    \State{$ \aA{G} \gets \funcName{RBBR}( \mA{C} )                                 $}  \Comment{get new bounded best resp.}
    \State{$ \aA{C} \gets \funcName{RBBR}( \mA{G} )                                                                                 $}
    \State \commentSymb{ Expected payoffs of these `tests' against mixture:}
    \State{$ u_{BRs} \gets \utA{G}(\aA{G},  \mA{C} ) + \utA{C}(\mA{G}, \aA{C}) $}  
    \If{$ u_{BRs} \leq 0 $}
        \State \textbf{break}
    \EndIf
        \State{$SG \gets \funcName{AugmentGame}(SG, \aA{G}, \aA{C})     $}
        \State{$\langle \mA{G}, \mA{C} \rangle \gets \funcName{SolveGame}(SG)                                                  $}
\EndWhile
\State{\textbf{return} $\langle \mA{G}, \mA{C} \rangle $} \Comment{found an BR-NE}
\end{algorithmic}

\protect\caption{\funcName{Parallel Nash Memory for GANGs with deterministic RBBRs}}
\label{alg:PNM}
\end{algorithm}

The algorithm is shown in Algorithm~\ref{alg:PNM}.
Intuitively, PNM incrementally grows a strategic game $SG$, over
a number of iterations, using the \funcName{AugmentGame} function.
It uses \funcName{SolveGame} to compute (via linear programming, see, e.g., \citet{Shoham08book}) 
a mixed strategy NE $\left\langle \mA G,\mA C\right\rangle $
\emph{of this smaller game} at the end of each iteration. 
In order to generate new candidate strategies to include in $SG$, at the beginning of each  
iteration the algorithm uses a `search' heuristic to deliver new promising strategies.
In our GANG setting, we use the resource-bounded best-response (RBBR)
functions of the players for this purpose. After having found new strategies,
we test if they `beat' the current $\left\langle \mA G,\mA C\right\rangle $, if they do, $u_{BRs} > 0$, and 
the game is augmented with these and solved again to find a new NE
of the sub-game $SG$. If they do not, $u_{BRs} \leq 0$, and the algorithm stops.\footnote{
Essentially, this is the same as the double oracle method, but it
replaces exact best-response computation with approximate RBBRs. Therefore, we refer to this 
as PNM, since that inherently allows for approximate/heuristic methods to deliver new `test' strategies~\citep{Oliehoek06GECCO}. It also enables us to develop a version for \emph{non-deterministic} RBBRs, by literally adopting the PNM algorithm.}

In order to augment the game, 
PNM evaluates (by simulation) each newly found strategy for each player
against all of the existing strategies of the other player, thus
constructing a new row and column for the maintained payoff matrix.

In order to implement the best response functions, any existing neural
network architectures can be used. 
However, we need to compute RBBRs against \emph{mixtures} of networks of the other
player.  
For $C$ this is trivial: we can simply generate a batch
of fake data from the mixture $\mA G$. 
Implementing an RBBR for $G$ against $\mA C$ is slightly more involved, as we
need to back-propagate the gradient from all the different $\aA C\in\mA C$ to
$G$. 
Intuitively, one can think of a combined network consisting of the $G$ network
with its outputs connected to every $\aA C\in\mA C$ (see
Figure~\ref{network_figure}). 
\begin{figure*}
    \center
	\includegraphics[width=0.8\textwidth]{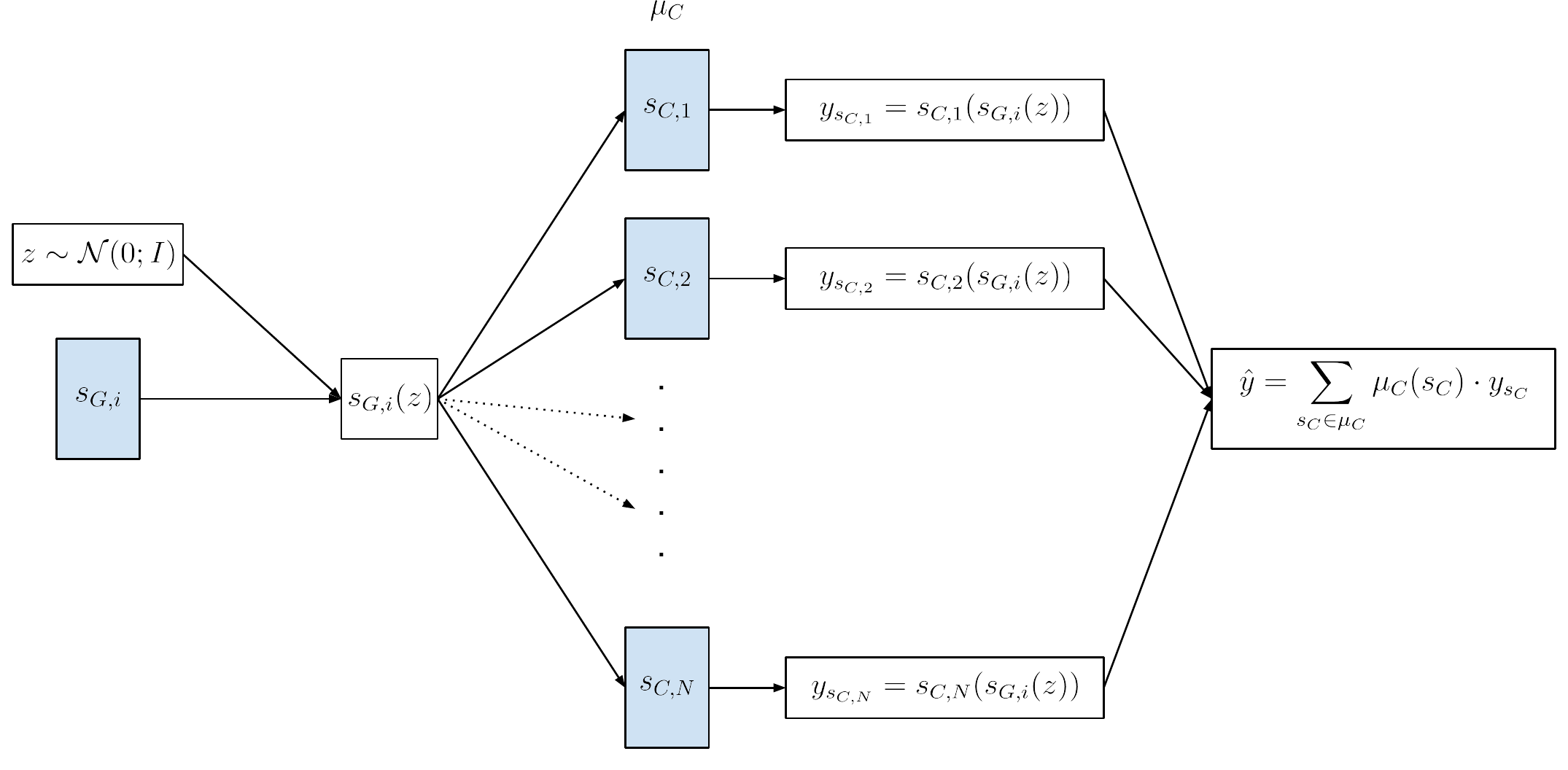}
	\caption{Network architecture that allows gradients from $\aA C\in\mA
	C$ to be back-propagated to $G$.}
	\label{network_figure}
\end{figure*}
The predictions $\hat{y}_{\aA C} $ 
of these components $\aA C\in\mA C$ are combined in a single linear output node 
$\hat{y} = 
\sum_{\aA C\in\mA C} 
\mA{C}(\aA{C}) \cdot \hat{y}_{\aA C}
$. 
This allows us to evaluate and backpropagate through the entire network. 
A practical implementation that avoids memory concerns instead loops through each
component $\aA C\in\mA C$ and does the evaluation  of the weighted prediction
$\mA{C}(\aA{C})  \cdot \hat{y}_{\aA C}$ and subsequent backpropagation per component.

Intuitively, it is clear that PNM converges to an RB-NE, which we now prove formally.
\begin{thm}
If PNM terminates, it has found an RB-NE. 
    \begin{proof}
\sloppypar
We show that $\utA{BRs}\leq0$ implies we have an RB-NE:
\begin{eqnarray}
\utA{BRs} & = & 
    \utA G(f_{G}^{RBBR}(\mA C),\mA C)   +   \utA C(\mA G,f_{C}^{RBBR}(\mA G))
    \nonumber\\
 & \leq & 0=\utA G(\mA G,\mA C)+\utA C(\mA G,\mA C)     \label{eq:UBR}
\end{eqnarray}
Note that, per Def.~\ref{def:RBBR}, 
$\utA G(f_{G}^{RBBR}(\mA C), \mA C) $   
$\geq$   
$\utA G(\aA{G}'            , \mA C) $
 for all computable 
$\aA{G}' \in  \aAS{G}^{RB}$ 
(and similar for $C$). 
Therefore, the only way that 
$\utA G(f_{G}^{RBBR}(\mA C),\mA C)$  
$\geq$
$\utA G(\mA G              ,\mA C)$ 
 could fail to hold,  is if $\mA G$ would include some strategies that are not
 computable (not in $\aAS{G}^{RB}$) that provide higher payoff. However, as the
 support of $\mA G$ is composed of 
    strategies computed in
 previous iterations, this cannot be the case.
 We conclude 
 $\utA G(f_{G}^{RBBR}(\mA C),\mA C)$  
 $\geq$    
 $\utA G(\mA G              ,\mA C)$
 and similarly
 $\utA C(\mA G, f_{C}^{RBBR}(\mA G))$  
 $\geq$     
 $\utA C(\mA G              ,\mA C)  $.
 Together with \eqref{eq:UBR} this directly implies 
 $\utA G(\mA G,\mA C) = \utA G(f_{G}^{RBBR}(\mA C),\mA C)$ and 
 $\utA C(\mA G,\mA C) = \utA C(\mA G,f_{C}^{RBBR}(\mA G))$, indicating we found
 an RB-NE.
\end{proof}
\end{thm}
\begin{cor}
    Algorithm~\ref{alg:PNM} terminates and 
\emph{monotonically} converges to an equilibrium.
\begin{proof}
    This follows directly from the fact that there are only finitely many 
    RBBRs and the fact that we never forget RBBRs that we computed before, thus
    the proof for PNM \citep{Oliehoek06GECCO} extends to Algorithm~\ref{alg:PNM}.
\end{proof}
\end{cor}
The PNM algorithm for GANGs is parameter free, but we mention 
two adaptations that are helpful.

\paragraph{Interleaved training of best responses.}
In order to speed up convergence of PNM, it is possible to train best responses
of $G$ and $C$ in parallel, giving $G$ access to the intermediate results of
$f_C^{RBBR}$.  This reduces the number of needed PNM iterations, but does not
seem to affect the quality of the found solutions, as demonstrated in 
Appendix \ref{app:interleaved}. 
The results shown in the main paper do not employ this trick.

\paragraph{Regularization of classifier best responses.}
In initial experiments, we found that the best-responses by $C$ tended to
overfit---see the supplement for an extensive analysis. In order to regularize
best responses for $C$, in each iteration, in the experiments
below we sample additional data points $x$ uniformly from a bounding box enclosing the current data (both real and fake), and
include these data points as additional fake data for training $C$. The number of 
`uniform fake points' chosen coincides with the batch size used for training. 

\section{Experiments}
\label{sec:experiments}
Here we report on experiments that aim to test if searching in mixed strategies
with PNM-GANG can help in reducing problems with training GANs,
and if the found solutions (near-RB-NEs) provide better generative models
and are potentially closer to 
true Nash equilibria than those found by GANs (near-LNEs).
Since our goal is to produce better generative models, we refrain from
evaluating these methods on complex data like images: image quality and log
likelihood are not aligned as for instance shown by \citet{Theis16ICLR}.
Moreover there is debate about whether GANs are overfitting (memorizing the
data) and assessing this from samples is difficult; only crude methods have been proposed 
e.g., \citep{AroraZ17,SalimansGZCRCC16,Odena17ICML,Karras18ICLR}, most of which provide merely a
measure of variability, not over-fitting.
As such, we choose to focus on irrefutable results on mixture of Gaussian tasks, 
for which the distributions can readily be visualized, and which themselves have 
many applications.
%

\paragraph{Experimental setup.} 
We compare our PNM approach (`PNM-GANG') to a vanilla GAN implementation and state-of-the-art MGAN \citep{Hoang18ICLR}. Detailed training
and architecture settings are summarized in the appendix.

The mixture components comprise grids and annuli with equal-variance
components, as well as non-symmetric cases with randomly located modes and with
a random covariance matrix for each mode.  For each domain we create test cases
with 9 and 16 components. In our plots, black points are real data, green
points are generated data.  Blue indicates areas that are classified as
`realistic' while red indicates a `fake' classification by~$C$.

\begin{figure*}[p]
\centering
\begin{tabular}{ccc}


\subfloat{\includegraphics[trim={1cm 0.5cm 0.5cm 1cm},clip, scale = 0.36]{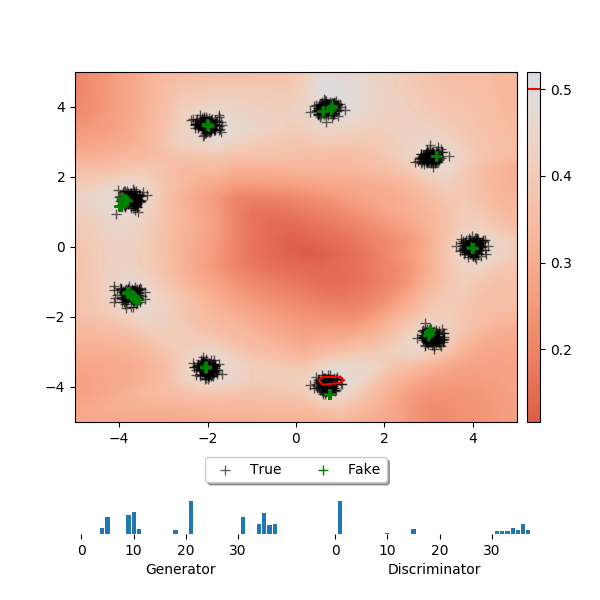}} &
\subfloat{\includegraphics[trim={1cm 0.5cm 0.5cm 1cm},clip, scale = 0.36]{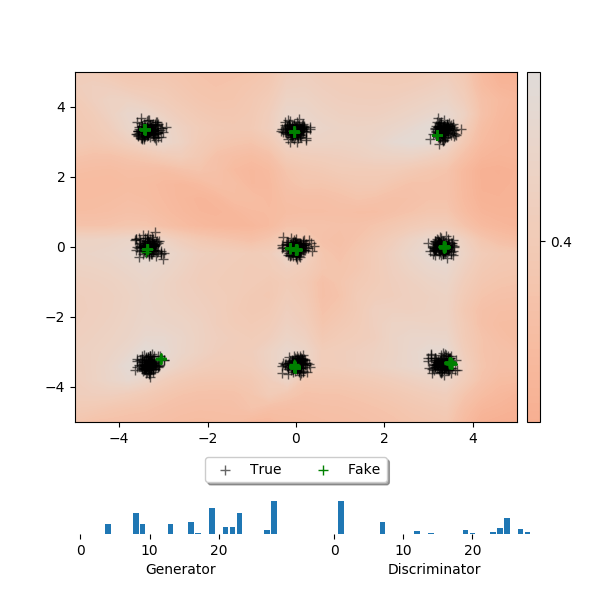}} &
\subfloat{\includegraphics[trim={1cm 0.5cm 0.5cm 1cm},clip, scale = 0.36]{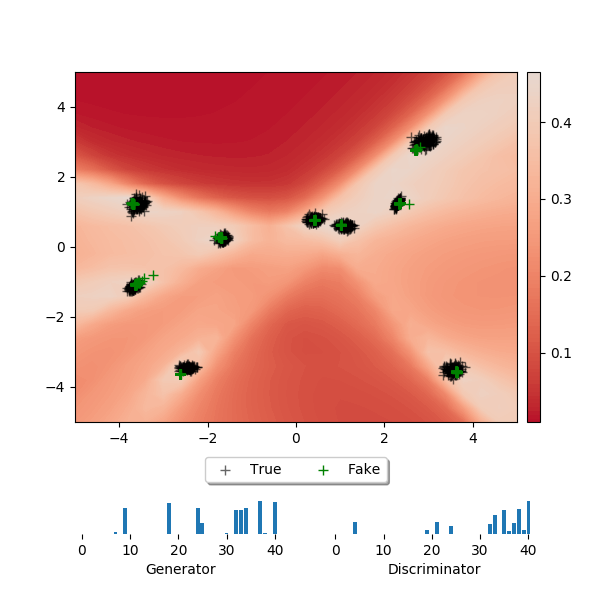}} \\


\subfloat{\includegraphics[trim={1cm 2.8cm 0.5cm 1cm},clip, scale = 0.36]{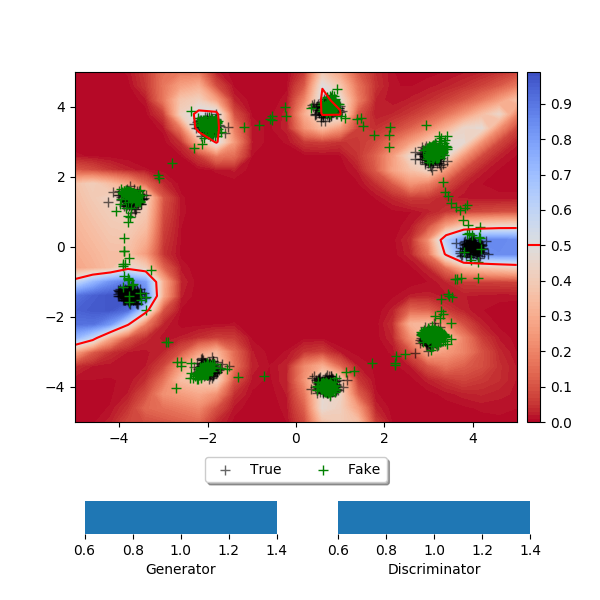}} &
\subfloat{\includegraphics[trim={1cm 2.8cm 0.5cm 1cm},clip, scale = 0.36]{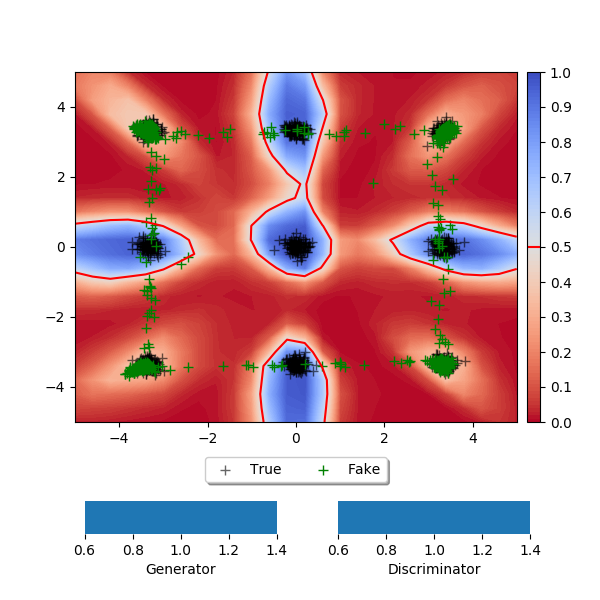}} &
\subfloat{\includegraphics[trim={1cm 2.8cm 0.5cm 1cm},clip, scale = 0.36]{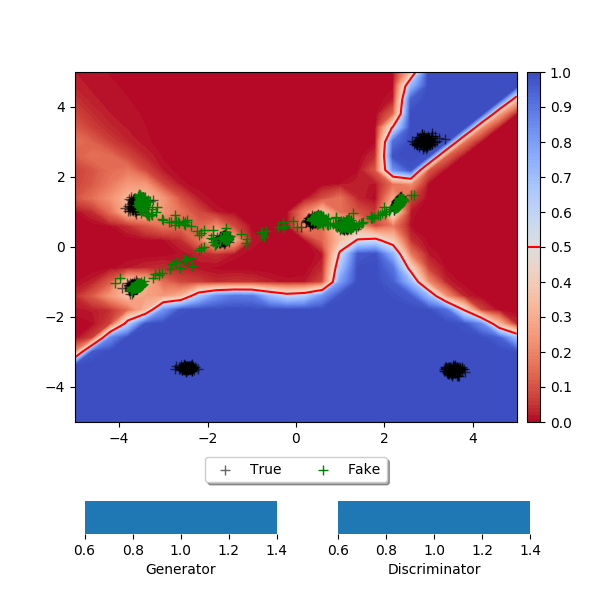}} \\


\subfloat{\includegraphics[trim={1cm 0.5cm 0.5cm 1cm},clip, scale = 0.36]{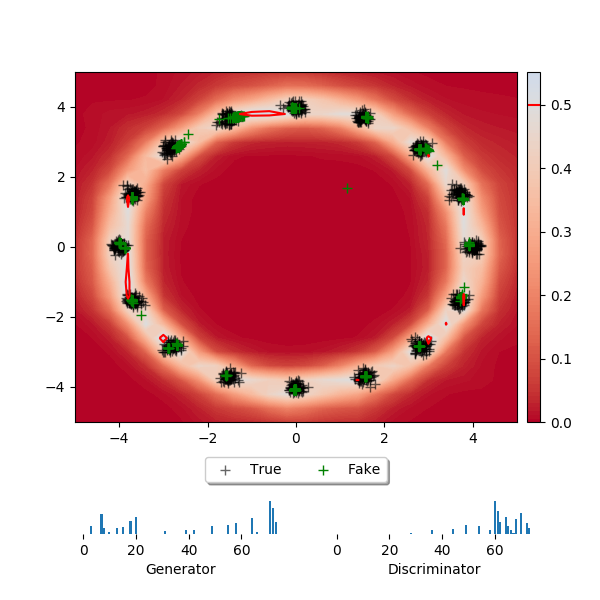}} &
\subfloat{\includegraphics[trim={1cm 0.5cm 0.5cm 1cm},clip, scale = 0.36]{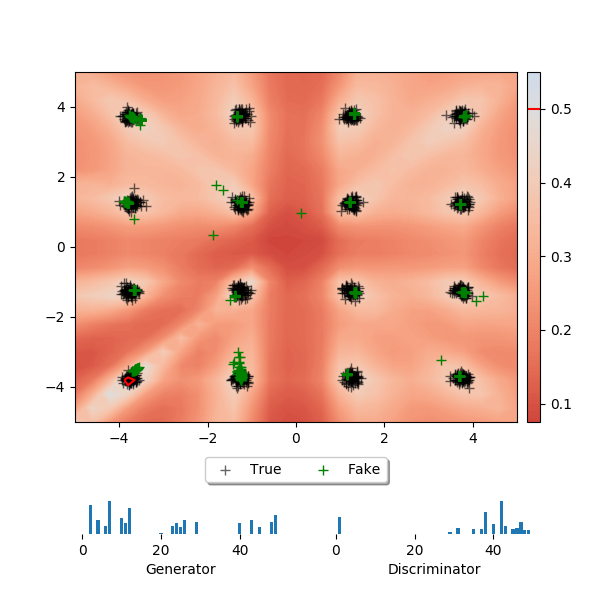}} &
\subfloat{\includegraphics[trim={1cm 0.5cm 0.5cm 1cm},clip, scale = 0.36]{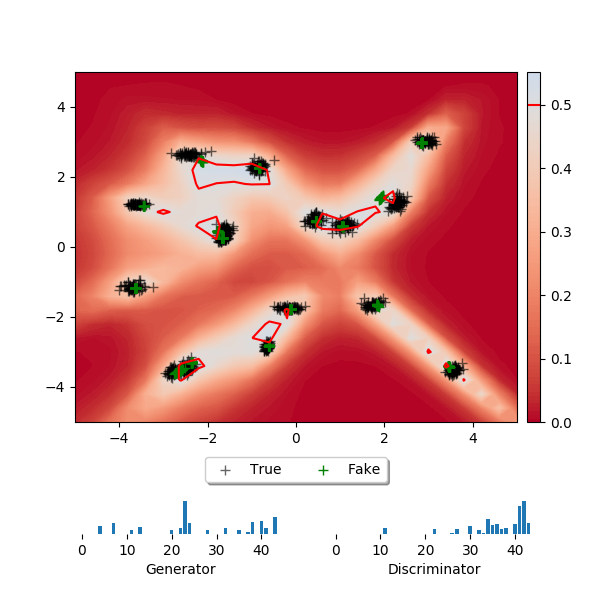}}\\

\subfloat{\includegraphics[trim={1cm 2.8cm 0.5cm 1cm},clip, scale = 0.36]{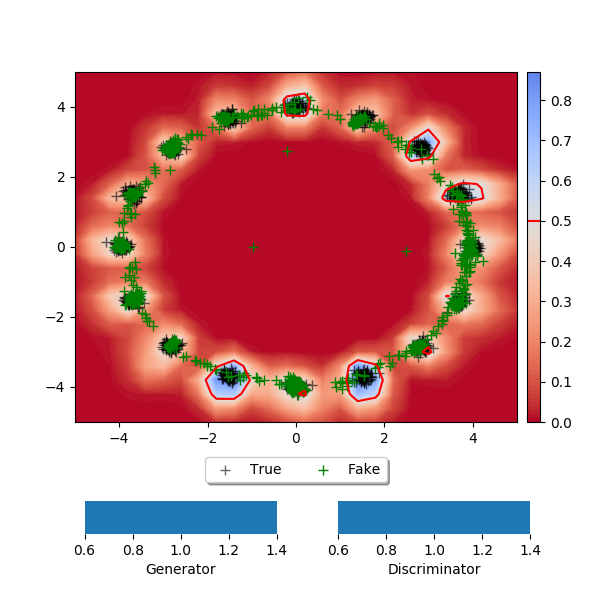}} &
\subfloat{\includegraphics[trim={1cm 2.8cm 0.5cm 1cm},clip, scale = 0.36]{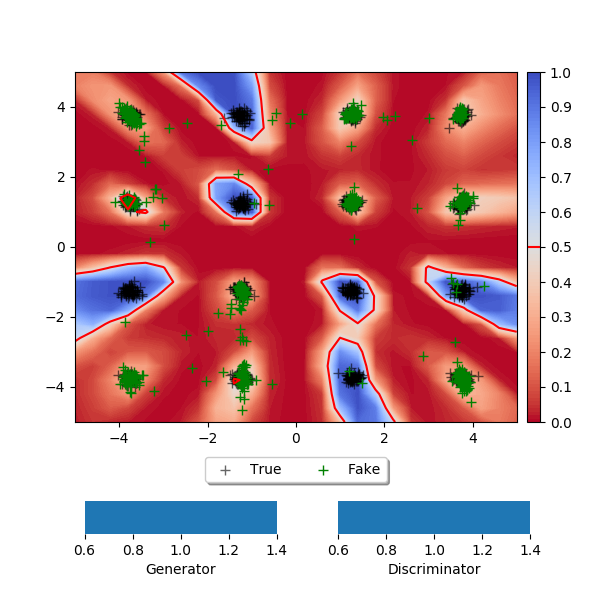}} &
\subfloat{\includegraphics[trim={1cm 2.8cm 0.5cm 1cm},clip, scale = 0.36]{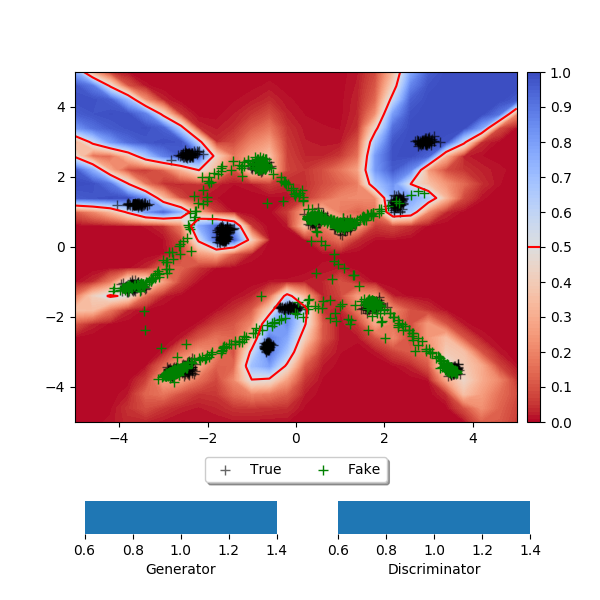}}

\end{tabular}

    \caption{Results for mixtures of Gaussians with 9 and 16 modes. Odd rows: PNM-GANG, Even rows: GAN. The histograms represent the probabilities in the mixed strategy of each player. True data is shown in black, while fake data is green. The classification boundary (where the classifier outputs 0.5) is indicated with a red line. Best seen in color.}
\label{fig:newresults}
\end{figure*}


\paragraph{Found solutions.} 
The results produced by regular GANs and PNM-GANGs are shown in
Figure~\ref{fig:newresults} and clearly convey three main points:
\begin{enumerate}
    \item
The PNM-GANG mixed classifier
has a much flatter surface than the classifier found by the GAN. Around the true
data, the GANG classifier outputs around 0.5 indicating indifference,
        which is {\em in line with the theoretical predictions} about the equilibrium
\citep{Goodfellow14NIPS27}. 
    \item
We see that
this flatter surface is not coming at the cost of inaccurate samples.
In contrast: nearly all samples shown are hitting one of the modes
        and thus {\em the PNM-GANG solutions are highly accurate}, 
        much more so than the GANs' solutions.
    \item
Finally, 
the PNM-GANGs, unlike GANs, 
        {\em do not suffer from mode omission}; they leave out no modes. 
We also note that PNM-GANG typically achieved these results with
fewer total parameters than the regular GAN, e.g., 1463 vs.\ 7653 for 
the random~9 task in Figure~\ref{fig:newresults}.
\end{enumerate}

This shows that, qualitatively, the use of multiple generators seems to lead to good results. 
This is corroborated by results we obtained using MGANs 
(illustrations and complete description can be found in Appendix \ref{sec:app:exploitability-results}), which only failed to represent one mode in the `random' task.

\paragraph{Impact of generator learning rate.}
The above results show that PNM-GANG can accurately cover multiple modes, however,
not all modes are \emph{fully} covered. 
As also pointed out by~\citet{ACB17_WGAN}, the best response of $G$
against $\mA C$ is a single point with the highest `realness', and therefore the
WGAN they introduced uses fewer iterations for $G$ than for $C$. Inspired by
this, we investigate if we can reduce the mode collapse by reducing the learning
rate of $G$. The results in Figure~\ref{fig:slowG} show
that more area of the modes are covered confirming this hypothesis.  However, it
also makes clear that by doing so, we are now generating some data outside of
the true data, so this is a trade-off. 
We point out that also with total mode degeneration (where modes collapse to
Dirac delta peaks),
the PNM mechanism theoretically would still converge, by adding in ever more
delta peaks covering parts of the modes.  

\paragraph{Exploitability of solutions.} 
Finally, to complement the above qualitative analysis, 
we also provide a quantitative analysis of the solutions
found by GANs, MGANs and PNM-GANGs. We investigate to what extent 
they are exploitable 
by newly introduced adversaries with some fixed computational power (as modeled
by the complexity of the networks we use to attack the found solution).

In particular, for a given solution $\tilde{\mA{}} = (\tilde{\mA G},\tilde{\mA C})$ we use the following measure of exploitability:
\begin{multline}
    expl^{RB}(\tilde{\mA G},\tilde{\mA C})	
\triangleq	
	\text{RBmax}_{\mA G}\utA G(\mA G,\tilde{\mA C}) \\
        +
        \text{RBmax}_{\mA C}\utA C(\tilde{\mA G},\mA C), 
\end{multline}
where `RBmax' denotes an approximate maximization performed by an adversary of
some fixed complexity.

\begin{figure}[tb]
\centering
\vspace{-5mm}
\scalebox{0.96}{
\begin{tabular}{ccc}
\subfloat{\includegraphics[trim={1cm 0.5cm 0.8cm 1cm},clip,scale=0.45]{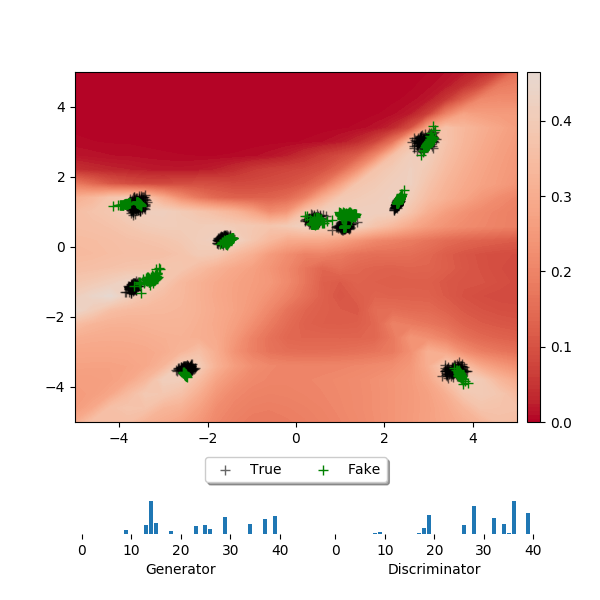}} 
    \\
\subfloat{\includegraphics[trim={1cm 0.5cm 0.5cm 1cm},clip, scale = 0.45]{dlog_random_9_5_5.png}} 
\end{tabular}
}\vspace{-4mm}
    \caption{Results for PNM with a learning rate of $7^{-4}$ for the generator (top). Compare with
    previous result (bottom, duplicated from the first row of Figure~\ref{fig:newresults}).
    }
\label{fig:slowG}
\end{figure}

\begin{figure*}[tb]
\centering
\includegraphics[scale=0.33]{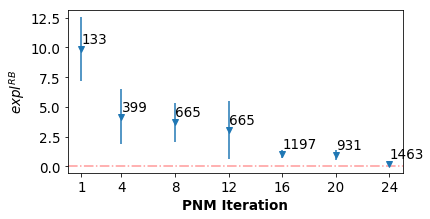} \hfill
\includegraphics[scale=0.33]{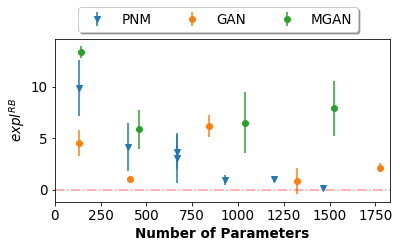} \hfill
\includegraphics[scale=0.33]{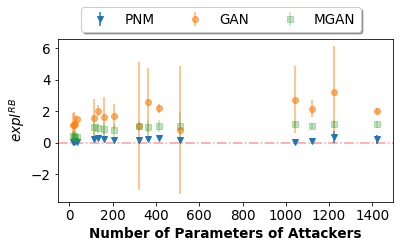}
\caption{Exploitability results for the nine randomly located modes task. 
    Left: exploitability of PNM-GANGs of various complexities (indicated numbers) found in different iterations.
    Middle: comparison of exploitability of PNM-GANGs, GANs and MGAN of various complexities.
    Right: exploitability of PNM-GANGs, GANs and MGAN, when varying the complexity of the attacking `RBmax' functions.
    Complexities are expressed in total number of parameters (for $G$ and $C$ together). See text for detailed explanation of all.
    }
\label{fig:exploitability}
\end{figure*}

That is, the `RBmax' functions are analogous to the $f_i^{RBBR}$ functions employed
in PNM, but the computational resources of `RBmax' could be different from those used for the $f_i^{RBBR}$.
Intuitively, it gives a higher score if $\tilde{\mA{}}$ is easier to exploit. 
However, it is not a true measure of distance to an equilibrium: it can return values that are lower
than zero which indicate that $\tilde{\mA{}}$ could not be exploited by the
approximate best responses.

Our exploitability is closely related to the use of GAN training metrics \citep{Im18ICLR}, but
additionally includes the exploitability of the classifier.
This is important: when only testing the exploitability of the generator, this does give a 
way to compare generators, {\em 
but it does not give a way to assess how far from equilibirum 
we might be}. 
Since finite GANGs are zero-sum games, distance to equilibrium 
is the desired performance measure. In particular, the exploitability of the
classifier actually may provide information about the quality of the generator:
if the generator holds up well against a perfect classifier, it should be
close to the data distribution. 
For a further motivation of this measure of exploitability, please see Appendix \ref{app:measureOfExploit}.

We perform two experiments on the 9 random modes task; results
for other tasks are in Appendix \ref{sec:app:exploitability-results}.
First, 
we investigate the
exploitability of solutions delivered by GANs, MGANs and GANGs of different
complexities (in terms of \emph{total} number of parameters used). 
For this, we compute `attacks' (approximate best responses) to these solutions 
using attackers of fixed complexity (a total of 453 parameters
for the attacking $G$ and $C$ together).

These results are shown in \fig\ref{fig:exploitability} (left and middle). The
left plot shows the exploitability of PNM-GANG after different numbers of iterations, as well as the
number of parameters used in the solutions found in those iterations
(a sum over all the networks in the support of the mixture). Error bars indicate
standard deviation.
It is apparent that PNM-GANG solutions with more parameters typically are less
exploitable. Also shown is that the variance of exploitability depends heavily on the
solution that we happen to attack.

The middle plot shows those same
results together with exploitability results we obtained for GANs and MGANs of different
complexities (all were trained to convergence). Note that here the
x-axis shows the complexity in terms of total parameters. 
The figure shows an approximately monotonic decrease in exploitability for GANGs for increasing number of parameters, while GANs and MGANs with higher complexity are still very exploitable. In contrast to GANGs, more complex architectures for GANs or MGANs are thus by no means a way to guarantee a better solution.

Secondly, 
we investigate what happens for 
the  converged GAN / PNM-GANG solution of \fig\ref{fig:newresults}, which have
comparable complexities,  when attacked with varying complexity attackers. 
We also employ an MGAN which has a significantly larger number of parameters (see Appendix \ref{sec:app:exploitability-results}). 
These results are shown in \fig\ref{fig:exploitability} (right). 
Clearly shown is that the PNM-GANG is robust with near-zero exploitability
even when attacked with high complexity attackers. The MGAN solution has a non-zero level of exploitability, roughly constant for several attacker complexities. 
In stark contrast, we see that the converged GAN solution is exploitable already
for low-complexity attackers, again suggesting that the GAN was stuck in an LNE far
away from a global NE.

Overall, these results demonstrate that GANGs can provide more robust solutions
than GANs/MGANs with the same number of parameters, suggesting that they are closer
to a Nash equilibrium and provide better generative models.

\section{Related work}
\label{sec:relatedWork}

\paragraph{Progress in zero-sum games.} 
\citet{BosanskyKLP14} devise a double-oracle algorithm for computing exact
equilibria in extensive-form games with imperfect information.
Their algorithm uses~\emph{best-response oracles}; PNM does so too, though in 
this paper using resource-bounded rather than exact best responses.
\citet{Lanctot17NIPS} generalize to non-exact sub-game routines.
Inspired by GANs, \citet{HazanSZ17} deal with general zero-sum settings with
non-convex loss functions. They introduce a weakening of local equilibria known
as \emph{smoothed local equilibria} and provide algorithms with guarantees on
the smoothed local regret. In contrast, we introduce a generalization of local
equilibrium (RB-NE) that allows for stronger notions of equilibrium, not only
weaker ones, depending on the power of the RBBR functions.
For the more restricted class of convex-concave zero-sum games, it was recently
shown that Optimistic Mirror Descent (a variant of gradient descent)
and its generalization Optimistic Follow-the-Regularized-Leader achieve faster
convergence rates than gradient descent~\citep{RakhlinS13a,RakhlinS13b}.
These algorithms have been explored in the context of GANs by~\cite{DaskISZ18ICLR}.
However, 
the convergence results do not apply as GANs are not convex-concave.

\paragraph{GANs.} 
The literature on GANs has been growing at an incredible rate, 
and a full overview of all the related works 
such as those by
\citet{ACB17_WGAN,ArjovskyB17,Huszar15_Exotic,NowozinCT16_FGAN,DaiABHC17_ICLR,ZhaoML16_EGAN_ICLR,AroraZ17,SalimansGZCRCC16,GulrajaniAADC17,Radford15arxiv}
is beyond the scop of this paper.
Instead we refer to \cite{Unterthiner18ICLR} for a reasonable comprehensive recent
overview.

\citet{Unterthiner18ICLR} introduce Coulomb GANs and show convergence for them
but only under the strong assumption that the ``generator samples can
move freely'' (which is not the case when training via gradient descent; samples
can only move small steps). 
Moreover, their approach essentially performs non-parametric density
estimation, which is based on the (Euclidean) distance between data points,
which we believe undermines one of the most attractive features of GANs
(not needing a distance metric). 

\citet{Karras18ICLR}, like PNM, incrementally grow the complexity of the
maintained models, but, unlike PNM, do not maintain a mixture.
Many of the techniques designed to improve GAN training, e.g.,
\citep{SalimansGZCRCC16,GulrajaniAADC17,Sonderby17ICLR,Creswell2018IEEESPM}, 
concern modifications to the update for one or both players. Such techniques 
can directly be used in GANGs by adapting the best response computation.


\paragraph{Explicit representations of mixtures of strategies.}
Recently, more researchers have investigated the idea of (more or less)
explicitly representing a set or mixture of strategies for the players.
For instance, \cite{ImMKT16} retains sets of networks that
are trained by randomly pairing up with a network for the other player thus
forming a GAN. This, like PNM, can be interpreted as a coevolutionary
approach, but unlike PNM, it does not have any convergence guarantees.

MAD-GAN \citep{Ghosh17arxiv} uses $k$ generators, but one discriminator. MGAN
\citep{Hoang18ICLR} proposes mixtures of $k$ generators, a classifier and a 
discriminator with weight sharing; and presents a theoretical analysis similar to
\citet{Goodfellow14NIPS27} assuming infinite capacity densities. Unlike PNM, none of these approaches have convergence guarantees.

Generally, explicit mixtures can bring advantages in two ways:
\emph{(1) Representation}: intuitively, a mixture of $k$ neural networks could better
represent a complex distribution than a single neural network of the same size, and
would be roughly on par with a single network that is $k$ times as big. \citet{Arora17ICML} show how to create such a bigger network that is
particularly suitable for dealing with multiple modes using a `multi-way
selector'. In our experiments we observed mixtures of
simpler networks leading to better performance than a single larger network of
the same total complexity (in terms of number of parameters). 
\emph{(2) Training}: Arora et al.\ use an architecture that is
tailored to representing a mixture of components and train a single such
network. We, in contrast, explicitly represent the mixture; given the
observation that good solutions will take the form of a mixture. This is a
form of domain knowledge that facilitates learning and convergence
guarantees. 

A closely related paper is the work by
\citet{Grnarova17arxiv}, which also builds upon game-theoretic tools to give certain
convergence guarantees. The main differences are as follows:

\begin{enumerate}
\item We provide a more general form of convergence (to an RB-NE) that is
applicable to \emph{all} architectures, that only depends on the power to compute best
responses, and show that PNM-GANG converges in this sense. We also show that if
agents can compute an $\epsilon$-best response, then the procedure converges to an $\epsilon$-NE.
\item \citet{Grnarova17arxiv} show that for a quite specific GAN architecture their first
algorithm converges to an $\epsilon$-NE. On the one hand, this result is an instantiation of our more
general theory: they assume they can compute exact (for $G$) and
$\epsilon$-approximate (for $C$) best responses; for such powerful players our
Theorem~\ref{thm:epsNashGivenSufficientResources} provides that guarantee. On
the other hand, their formulation works without discretizing the spaces of
strategies.
\item The practical implementation of their algorithm does not provide guarantees. 
\end{enumerate}

Finally, \citet{Ge18arxiv} propose a method similar to ours that uses \emph{fictitious
play} \citep{Brown1951iterative,Fudenberg98book} rather than PNM. Fictitious play does not explicitly model mixed strategies for the
agents, but interprets the opponent's historical behavior as such a mixed strategy. The
average strategy played by the `Fictitious GAN' approach converges to a Nash equilibrium
\emph{assuming that ``the discriminator and the generator are updated according to the
best-response strategy at each iteration''}, which follow from the result by
\cite{Danskin1981} which states that fictitious play converges in continuous zero-sum
games. Intuitively, fictitious play, like PNM, in each iteration only ever touches a
finite subset of strategies, and one can show that the value of such subgames converges.
While this result gives some theoretical underpinning to Fictitious GAN, of course in
practice the assumption is hard to satisfy and the notion of RB-NE that we propose may
apply to analyze their approach too. 
Also, in their empirical results they limit the history of actions (played neural networks
in previous iterations) to 5 to improve scalability at the cost of convergence
guarantees.
The Fictitious GAN is not explicitly shown to be more
robust than normal GANs, as we show in this paper, but it is demonstrated to produce high
quality images, thus showing the potential of game theoretical approaches to GANs to
scale.

\paragraph{Bounded rationality.}
The proposed notion of RB-NE is one of bounded rationality
\citep{Simon55behavioral}.
Over the years a number of different such notions have been proposed, e.g.,
see \citet{Russell97AIJ,Zilberstein11metareasoning}. 
Some of these also target agents in games. Perhaps the most
well-known such a concept is the quantal response equilibrium
\citep{McKelvey95GEB}. Other concepts take into account an explicit cost of
computation \citep{Rubinstein86JEC,Halpern14TopiCS}, or explicitly limit the
allowed strategy, for instance by limiting the size of finite-state
machines that might be employed \citep{Halpern14TopiCS}.
However, these notions are motivated to explain \emph{why} people might show
certain behaviors or \emph{how} a decision maker should use its
limited resources. We on the other hand, take the why and how of bounded
rationality as a given, and merely model the outcome of a resource-bounded
computation (as the computable set $\aAS i^{RB}\subseteq\aAS i$). In other
words, we make a minimal assumption on the nature of the
resource-boundedness, and aim to show that even under such general assumptions we
can still reach a form of equilibrium, an RB-NE, of which the quality can
be directly linked
(via Theorem~\ref{thm:epsNashGivenSufficientResources}) 
to the computational power of the agents.


\section{Conclusions}

\label{sec:conclusions}

We introduce finite GANGs---Generative Adversarial Network Games---a novel
framework for representing adversarial networks by formulating them as finite
zero-sum games. By tackling them with techniques working in mixed strategies
we can avoid getting stuck in local Nash equilibria (LNE).
As finite GANGs have extremely large strategy spaces we cannot expect to
exactly (or $\epsilon$-approximately) solve them. Therefore,
we introduced the
resource-bounded Nash equilibrium (RB-NE). 
This notion is richer than LNE in that it captures not only failures of
escaping local optima of gradient descent, but applies to any approximate 
best-response computations, including methods with random restarts. 

Additionally, GANGs can draw on a rich set of methods for solving zero-sum
games~\citep{Fudenberg98book,RakhlinS13b,Foster16NIPS,Oliehoek06GECCO}.
In this paper, we build on PNM and prove that the resulting method monotonically
converges to an RB-NE.
We empirically demonstrate that the resulting method does not suffer from
typical GAN problems such as mode collapse and forgetting.
We also show that the GANG-PNM solutions are closer to theoretical predictions, and
are less exploitable than normal GANs: by using PNM we can train models that are
more robust than GANs of the same total complexity, indicating they are closer to a
global Nash equilibrium and yield better generative performance.

\paragraph{Future work} 
We presented a framework that can have many
instantiations and modifications. 
For example, one direction is to employ different learning algorithms.
Another direction could focus on modifications of PNM, such as to allow
discarding ``stale'' pure strategies, which
would allow the process to run for longer without being inhibited by the size
of the resulting zero-sum ``subgame'' that must be maintained and repeatedly
solved.
The addition of fake uniform data as a guiding component suggests
that there might be benefit of considering ``deep-interactive learning''
where there is deepness in the number of players that interact in
order to give each other guidance in adversarial training. This could
potentially be modelled by zero-sum polymatrix games~\citep{CaiCDP16}.

\section*{Acknowledgments}
This research made use of a GPU donated by NVIDIA. 
F.A.O.\ is funded by EPSRC First Grant EP/R001227/1, and ERC Starting
Grant \#758824\textemdash INFLUENCE.

\bibliographystyle{abbrvnat}
\bibliography{gangs.bib}

\appendix
\onecolumn

\section{The zero-sum formulation of GANs and GANGs}
\label{app:zerosum}

In contrast to much of the GAN-literature,
we explicitly formulate GANGs as being zero-sum games. GANs \citep{Goodfellow14NIPS27}
formulate 
the payoff of the generator as a function of the fake data only: 
$\utA G = \E_{z\sim p_{z}}
\left[
    \mf
    \left(  1  -   \aA{C}(\aA G(z))   \right)
\right]$. 
However, it turns out that this difference typically has no implications
for the sought solutions. We clarify this with the following theorem,
and investigate particular instantiations below. In game theory, 
two games are called \emph{strategically equivalent} if
they possess exactly the same set of Nash equilibria; this (standard)
definition is concerned only about the mixed strategies played in equilibrium
and not the resulting payoffs. The following is a well-known game
transformation (folklore, see \citet{Liu96}) that creates a new strategically
equivalent game:
\begin{fact}
\label{fact:equiv}Consider a game $\Gamma$=$\left\langle \{1,2\},\{\aAS 1,\aAS 2\},\{\utA 1,\utA 2\}\right\rangle $.
Fix a pure strategy $s_{2}\in\aAS 2.$ Define $\bar{{\utA 1}}$ as
identical to $\utA 1$ except that $\bar{\utA 1}(\aA{i},\aA 2)=\utA 1(\aA{i},\aA 2)+c$
for all $\aA{i}\in\aAS 1$ and some constant $c$. 
We have that $\Gamma$ and 
$\bar{{\Gamma}}$=$\left\langle \{1,2\},\{\aAS 1,\aAS 2\},\{\bar{\utA 1},\utA 2\}\right\rangle $
are strategically equivalent.
\end{fact}
\begin{thm}
\label{thm:fake-real-decomp}
Let $Fake_{C}(\aA G,\aA C)$, and $Real_{C}(\aA C)$ be arbitrary functions, then 
any finite (non-zero-sum) two-player
game between $G$ and $C$ with payoffs of the following form:
\[
\utA G=Fake_{G}(\aA G,\aA C) \triangleq -Fake_{C}(\aA G,\aA C),
\]
\[
\utA C=Real_{C}(\aA C)+Fake_{C}(\aA G,\aA C),
\]
is strategically equivalent to a zero-sum game where $G$ has instead payoff
$\bar{\utA G}\defas-Real_{C}(\aA C)-Fake_{C}(\aA G,\aA C).$\end{thm}
\begin{proof}
By adding$-Real_{C}(\aA C)$ to $G$'s utility function, for each pure strategy $\aA C$ of $C$ 
we add a different constant
to all utilities of $G$ against $\aA C$. Thus, by applying
Fact \ref{fact:equiv} iteratively for all $\aA C\in\aAS C$ we see
that we produce a strategically equivalent game.
\end{proof}

Next, we formally specify the conversion of existing GAN models
to GANGs. We consider the general measure function 
that covers GANs ($\mf(x)=\log(x)$)  and WGANs ($\mf(x)=x-1$). In these models, the payoffs are specified
as
\[
\utA G(\aA G,\aA C)\defas-\E_{z\sim p_{z}}\left[\mf\left(1-\aA C\left(\aA G(z)\right)\right)\right],
\]
\[
\utA C(\aA G,\aA C)\defas\E_{x\sim p_{d}}\left[\mf\left(\aA C(x)\right)\right]+\E_{z\sim p_{z}}\left[\mf\left(1-\aA C\left(\aA G(z)\right)\right)\right].
\]
{\sloppy
These can be re-written using 
$
Fake_{C}(\aA G,\aA C)
=
\E_{z\sim p_{z}}\left[\mf\left(1-\aA C\left(\aA G(z)\right)\right)\right]
$
and 
$Real_{C}(\aA C)
=
\E_{x\sim p_{d}}\left[\mf\left(\aA C(x)\right)\right]
$.
This means that we can employ \rthm\ref{thm:fake-real-decomp} and
equivalently define 
a GANG with zero-sum payoffs that preserves the NEs.
}

In practice, most work on GANs uses a different
objective, introduced by \cite{Goodfellow14NIPS27}. They say that {[}formulas
altered{]}: 
\begin{quote}
``Rather than training $G$ to minimize $\log(1-\aA C(\aA G(z)))$ we
can train $G$ to maximize $\log\aA C(\aA G(z))$. This objective function
results in the same fixed point of the dynamics
of $G$ and $C$ but provides much stronger gradients early in learning.'' 
\end{quote}
This means that they redefine 
$\utA G(\aA G,\aA C)\defas\E_{z\sim p_{z}}\left[\mf\left(\aA C\left(\aA G(z)\right)\right)\right],$
which still can be written as 
$Fake_{G}(\aA G,\aA C)=\E_{z\sim p_{z}}\left[\mf\left(\aA C\left(\aA
G(z)\right)\right)\right]$, which means that it is candidate for transformation
to $\bar{\utA G}$.
Now, \emph{as long as the classifier's payoff is also adapted}
we can still write the payoff functions in the form of \rthm\ref{thm:fake-real-decomp}.
That is, the trick is compatible with a zero-sum formulation, as
long as it is also applied to the classifier. This then yields:
\begin{equation}
    \utA C(\aA G,\aA C) =
        \E_{x\sim p_{d}}
        [\mf(\aA C(x))]
        -
        \E_{z\sim p_{z}}
        [\phi(\aA C(\aA G(z)))].
\end{equation}

\section{Experimental setup}
\label{app:expsetup}
Table \ref{tab:settings} summarizes the settings for GAN and PNM training.
As suggested by \citep{ganhacks}, we use leaky ReLU as inner activation
for our GAN implementation to avoid sparse gradients. Generators
have linear output layers. Classifiers use sigmoids for the final layer. Both classifiers and generators are multi-layer perceptrons with 3 hidden layers.
We do not use techniques such as Dropout or Batch Normalization, as they did not
yield significant improvements in the quality of our experimental results. The MGAN configuration is identical to that of Table 3 in Appendix C1 of \citet{Hoang18ICLR}.

\begin{table}[ht]
\centering
\scalebox{0.8}{
\begin{tabular}{c|c|c|c|}
\cline{2-3}
 & \textbf{GAN} & \textbf{RBBR} \\ \hline
\multicolumn{1}{|c|}{\textit{Learning Rate}} & $3\cdot 10^{-4}$ & $5 \cdot 10^{-3}$ \\ \hline
\multicolumn{1}{|c|}{\textit{Batch Size}} & $128$ & $128$ \\ \hline
\multicolumn{1}{|c|}{\textit{Z Dimension}} & $40$ & $5$ \\ \hline
\multicolumn{1}{|c|}{\textit{H Dimension}} & $50$ & $5$ \\ \hline
\multicolumn{1}{|c|}{\textit{Iterations}} & $20000$ & $750$ \\ \hline
\multicolumn{1}{|c|}{\textit{Generator Parameters}} & $4902$ & $92$ \\ \hline
\multicolumn{1}{|c|}{\textit{Classifier Parameters}} & $2751$ & $61$ \\ \hline
\multicolumn{1}{|c|}{\textit{Inner Activation}} & Leaky ReLU & Leaky ReLU \\ \hline
\multicolumn{1}{|c|}{\textit{Measuring Function}} & $\log$ & $10^{-5}$-bounded $\log$ \\ \hline
\end{tabular}
}\\
~
\caption{Settings used to train GANs and RBBRs.}
\label{tab:settings}
\end{table}

\subsection{Using uniform fake data to regularize classifier best responses}
\label{app:unifake}
\begin{figure*}[tb]
\centering
\begin{tabular}{ccccc}
\hspace{-0.5cm}
\subfloat{\includegraphics[scale = 0.26]{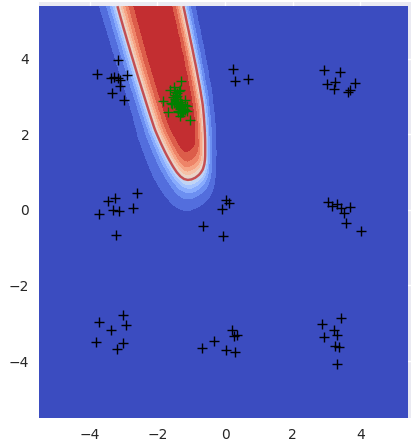}} &
\subfloat{\includegraphics[scale = 0.26]{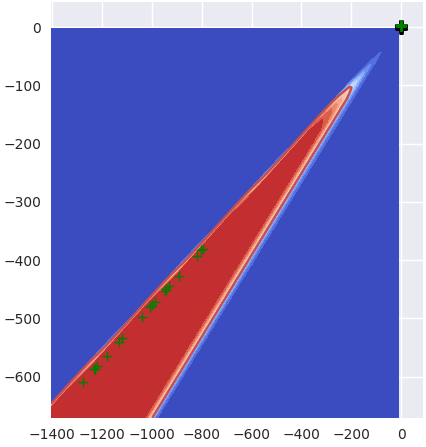}} &
\subfloat{\includegraphics[scale = 0.26]{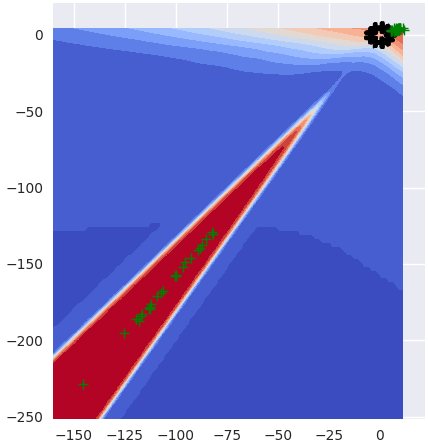}} &
\subfloat{\includegraphics[scale = 0.26]{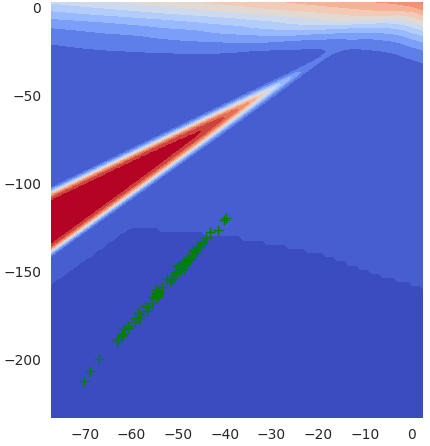}} &
\subfloat{\includegraphics[scale = 0.26]{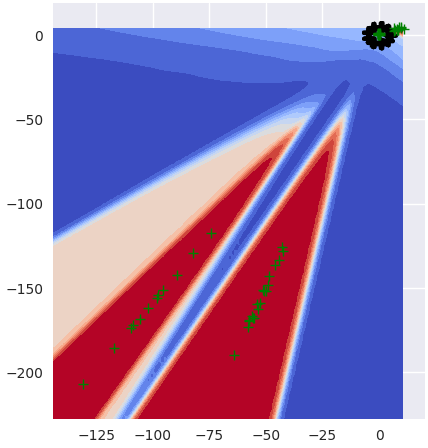}} 
\end{tabular}
\caption{Illustration of exploitation of `overfitted' classifier best-response.}
\label{fig:overfitted}
\end{figure*}

\paragraph{Plain application of PNM.}
In GANGs, $G$ informs $C$ about what good strategies are and vice versa.
However, as we will make clear here, this $G$
has limited incentive to provide the best possible training signal
to $C$.

This is illustrated in \fig~\ref{fig:overfitted}. 
The left two figures 
show the \emph{same} best response by $C$: zoomed in on the
data and zoomed out to cover some fake outliers. Clearly, $C$ needs
to find creative solutions to try and both remove the faraway
points and also do good near the real data. As a result, it ends up with
the shown narrow beams in an effort to give high score to the true
data points (a very broad beam would lower their scores), but this
exposes $C$ to being exploited in later iterations: $G$ needs to merely
shift the samples to some other part of the vast empty space around
the data. 

This phenomenon is nicely illustrated by the remaining three
plots  (that are from a different training run, but illustrate it well):
the middle plot shows an NE that targets one beam, this is exploited
by $G$ in its next best response (fourth image, note the different scales
on the axes, the `beam' is the same). The process continues, and
$C$ will need to find mixtures of all these type of complex counter
measures (rightmost plot). This process can take a long time.

\paragraph{PNM with added uniform fake data.} 
The GANG formalism
allows us to incorporate a simple way to resolve this
issue and make training more effective. In each iteration, we look
at the total span (i.e., bounding box) of the real and fake data combined,
and we add some uniformly sampled fake data in this bounded
box (we used the same amount as fake data produced by $G$). In that
way, we further guide $C$ in order to better guide the
generator (by directly making clear that all the area beyond the true
data is fake). The impact of this procedure is illustrated by \fig~\ref{fig:unifakeconv},
which shows the payoffs that the maintained mixtures $\mA G,\mA C$
achieve against the RBBRs computed against them (so this is a measure
of security), as well as the `payoff for tests' ($\utA{BR}$). 
Clearly, adding uniform fake data leads
to much faster convergence. 
As such, we perform our main comparison to GANs with this uniform
fake data component added in. {
\newlength{\mywOneB}
\setlength{\mywOneB}{0.7\columnwidth}
\begin{figure}[tb]
\centering
\scalebox{0.65}{
\begin{tabular}{cc}
    \hspace{-8mm}
\subfloat{\includegraphics[scale = 0.5]{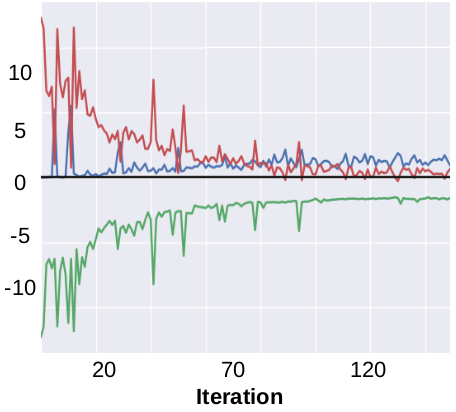}} &
\subfloat{\includegraphics[scale=0.5]{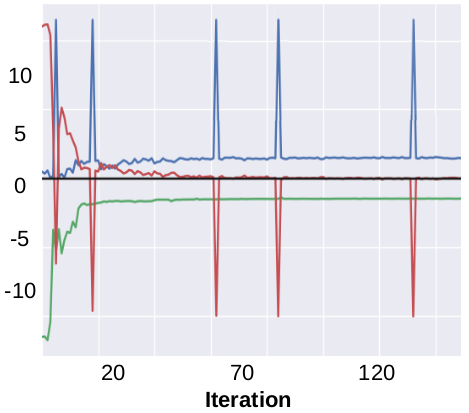}} 
\end{tabular}
}
\caption{Convergence without (left) and with (right) adding uniform fake data.
    Shown is payoff as a function of the number of iterations of PNM: generator (blue), classifier (green), tests (red).
    The tests that do not generate positive payoff (red line $< 0$) are not 
    added to the mixture.
}
\label{fig:unifakeconv}
\end{figure}
}

\subsection{A measure of exploitability}
\label{app:measureOfExploit}


GANGs are guaranteed to converge to an RB-NE, while GANs (in the best
case) only converge to a local NE (LNE). Therefore we hypothesize
that GANGs, given sufficient computational power, yield better solutions
that are `closer to a Nash equilibrium'. This is important, even
in settings where we only care about the performance of the generator:
\emph{by formulating the problem as a game (GAN or GANG),} \emph{we
have committed to approach the problem by trying to find an NE} (even
if that is computationally intractable).

Here we propose a measure to test how far from a Nash equilibrium
we might be by measuring to what extent the found solution is exploitable
by adversaries of some complexity. While this is far from a perfect
measure, ourselves only being equipped with bounded computational
resources, this might be the best we can do.

In particular, recall that at an NE, we realize the value of the game:

\begin{equation}
\min_{\mA C}\max_{\mA G}\utA G(\mA G,\mA C)=\max_{\mA G}\min_{\mA C}\utA G(\mA G,\mA C)=\gameval
\end{equation}

In other words, for an equilibrium $\left\langle \mA G^{*},\mA C^{*}\right\rangle $
we have that:
\begin{equation}
\max_{\mA G}\utA G(\mA G,\mA C^{*})=\gameval
\end{equation}

\begin{equation}
\min_{\mA C}\utA G(\mA G^{*},\mA C)=\gameval=\min_{\mA C}-\utA C(\mA G^{*},\mA C)=-\max_{\mA C}\utA C(\mA G^{*},\mA C)
\end{equation}
which means that $\max_{\mA C}\utA C(\mA G^{*},\mA C)=-\gameval$.
It also means that if we would know $\gameval$, we would be able
to say `how far away' some found $\tilde{\mA C}$ is from an equilibrium
strategy $\mA C^{*}$ (in terms of value) by looking at the \emph{best-response
exploitability}\textbf{:}
\begin{equation}
expl^{BR}(\tilde{\mA C})=\max_{\mA G}\utA G(\mA G,\tilde{\mA C})-\gameval\geq0
\end{equation}
where the inequality holds assuming we can find the actual best response.
Similarly for a $\tilde{\mA G}$ we can have a look at
\[
\min_{\mA C}\utA G(\tilde{\mA G},\mA C)-\gameval\leq0
\]
I.e., if $\tilde{\mA G}$ is an equilibrium strategy $\min_{\mA C}\utA G(\tilde{\mA G},\mA C)=\gameval$
but if $\tilde{\mA G}$ is less robust, the minimum is going to be
lower leading to a negative value of the expression (again assuming
perfect best response $\mA C$). Since negative numbers are unnatural
to interpret as a distance, we define (again, assuming perfect best response computation):
\begin{equation}
expl^{BR}(\tilde{\mA G})=\gameval-\min_{\mA C}\utA G(\tilde{\mA G},\mA C)=\gameval+\max_{\mA C}-\utA G(\tilde{\mA G},\mA C)=\gameval+\underbrace{\max_{\mA C}\utA C(\tilde{\mA G},\mA C)}_{\text{at least \ensuremath{-\gameval}}}\geq0
\end{equation}
Note that, expanding the payoff function, and realizing that the maximum will
be attained at a deterministic strategy $\aA C$:
\[
expl^{BR}(\tilde{\mA G})=\gameval+\underbrace{\max_{\aA C}\left(\mathbf{E}_{x\sim p_{d}}\left[\mf(\aA C(x))\right]+\mathbf{E}_{z\sim p_{z},z\sim\tilde{\mA G}(z)}\left[\mf(1-\aA C(x))\right]\right)}_{\text{`training divergence or metric'}},
\]
which means this exploitability corresponds ``to optimizing the divergence
and distance metrics that were used to train the GANs'' \citep{Im18ICLR},
offset by $\gameval$.

In the theoretical non-parametric setting (with infinite capacity
densities as the strategies) this is directly useful, because then
we know that the generator is able to exactly match the density, and
that $\gameval=\log4$ (for the $\log$ measuring function, which
leads to a correspondence to the Jensen-Shannon divergence \citep{Goodfellow14NIPS27}),
since $C$ will not be able to discriminate (give higher `realness
score' to) real data from on fake data.

However, this may not be a score that is attainable in finite GANG
or GANs: First, there might not be a (mixture of) neural networks
in the considered class that will match the data density, second there
might not be classifier strategies that perfectly classify this discrepancy
between the true and generated fake data. Which means that we can
not deduce $\gameval$. Therefore, even when we assume we can compute
perfect best responses, in the finite setting computing these distances
might not be sufficient to determine how far from a Nash equilibrium
we are. Like \citet{Im18ICLR} and others before them, one might hope
that using a neural network classifier to compute a finite approximation
to the aforementioned divergences and metrics., i.e., approximating
$expl^{BR}(\tilde{\mA G})$, will still give a way to compare generators
(and it does), but it does not tell how far from equilibrium one is.

However, even though we do not know $\gameval$, \emph{in the case
where we can compute perfect best responses}, we can compute a notion
of distance of a tuple $\left\langle \tilde{\mA G},\tilde{\mA C}\right\rangle $
to equilibrium by looking at the sum:

\begin{multline}
expl^{BR}(\tilde{\mA G},\tilde{\mA C})\triangleq expl^{BR}(\tilde{\mA C})+expl^{BR}(\tilde{\mA G})=\\
\left(\max_{\mA G}\utA G(\mA G,\tilde{\mA C})-\gameval\right)+\left(\gameval+\max_{\mA C}\utA C(\tilde{\mA G},\mA C)\right)=\max_{\mA G}\utA G(\mA G,\tilde{\mA C})+\max_{\mA C}\utA C(\tilde{\mA G},\mA C)\label{eq:perfect-BR-distance-to-nash}
\end{multline}
So by reasoning about the tuple $(\tilde{\mA G},\tilde{\mA C})$ rather
than only $\tilde{\mA G}$, we are able to eliminate the factor of
uncertainty: $\gameval$.

We therefore propose to take this approach, also in the case where
we cannot guarantee computing best responses (the second factor of
uncertainty). However, note that in (\ref{eq:perfect-BR-distance-to-nash}),
since both terms are guaranteed to be larger than 0, we do not need
to worry about cancellations of terms and the measure will never underestimate.
This is no longer the case when using approximate maximization. Still
we can define a \emph{resource-bounded} variant:

\begin{eqnarray}
expl^{RB}(\tilde{\mA G},\tilde{\mA C}) & \triangleq & \left(\text{RBmax}_{\mA G}\utA G(\mA G,\tilde{\mA C})-\gameval\right)+\left(\gameval+\text{RBmax}_{\mA C}\utA C(\tilde{\mA G},\mA C)\right)\label{eq:RBBR-dist-intermediate}\\
& = & \text{RBmax}_{\mA G}\utA G(\mA G,\tilde{\mA C})+\text{RBmax}_{\mA C}\utA C(\tilde{\mA G},\mA C).\label{eq:RBBR-distance-to-nash}
\end{eqnarray}
It does not eliminate the second source of uncertainty, but neither
does approximating $expl^{BR}(\tilde{\mA G})$. This is a perfectly
useful tool (a lower bound to be precise) to get some information
about the distance to an equilibrium, as long as we are careful with
its interpretation. In particular, since either or both of the terms
of (\ref{eq:RBBR-dist-intermediate}) could end up being lower than
0 (due to failure of computing an exact best response), we might end
up underestimating the distance, and we can even get negative values.
However, $expl^{RB}$ is still useful for \emph{comparing} different
found solution pairs $\left\langle \tilde{\mA G},\tilde{\mA C}\right\rangle $
and $\left\langle \tilde{\mA G}',\tilde{\mA C}'\right\rangle $ as
long as we use the same computational resources to compute approximate
best responses against them. Negative values of $expl^{RB}(\tilde{\mA G},\tilde{\mA C})$
should be interpreted as ``robust up to our computational resources
to attack it''.

\section{Additional empirical results}

\subsection{Low generator learning rate}

\fig\ref{fig:slowGfull} shows the results of a lower learning rate for the
generator for all 9 mode tasks. The general picture is the same as for the
random 9 modes task treated in the main paper: the modes are better covered by
fake data, but in places we see that this is at the expense of accuracy. 

\begin{figure}[tb]
\centering
\vspace{-6mm}

\scalebox{0.96}{
\begin{tabular}{ccc}
\hspace{-7mm}
\subfloat{\includegraphics[width = 0.35\columnwidth]{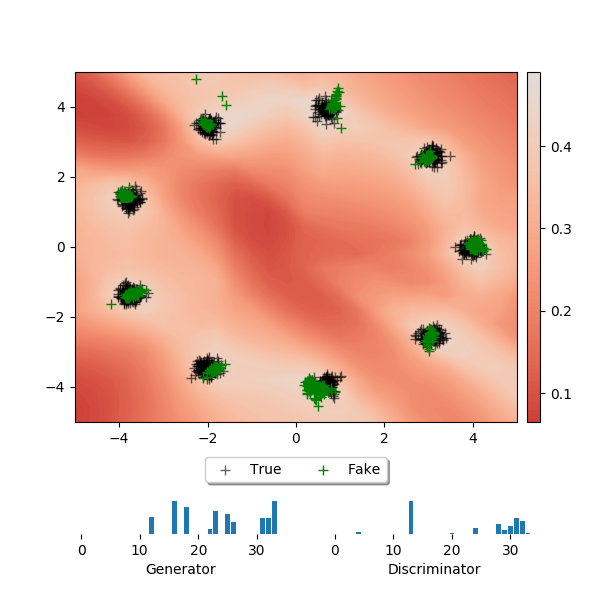}} &
\subfloat{\includegraphics[width = 0.35\columnwidth]{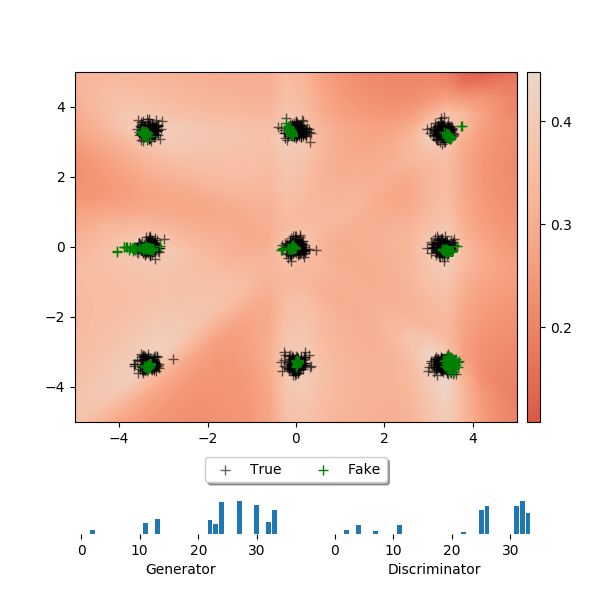}}
\subfloat{\includegraphics[width = 0.35\columnwidth]{rand}} &
\end{tabular}
}\vspace{-4mm}
\caption{Results for PNM with a learning rate of $7^{-4}$ for the generator. Compare with
the first row of Figure~\ref{fig:newresults}.}
\label{fig:slowGfull}
\end{figure}

\subsection{Exploitability results}
\label{sec:app:exploitability-results}

MGAN \cite{Hoang18ICLR} proposes a setup with a mixture of $k$ generators, a
classifier, and 
a discriminator. In their setting, the generator mixture aims
to create samples which match the training data distribution, while the
discriminator distinguishes real and generated samples, and the classifier
tries to determine which generator a sample comes from.  Similar to
\citet{Goodfellow14NIPS27},  MGAN presents a theoretical analysis assuming
infinite capacity densities. We use MGAN as a state-of-the art baseline that
was explicitly designed to overcome the problem of mode collapse.

Figure \ref{fig:mgan_samples} shows the results of MGAN on the mixture of Gaussian
tasks. 
MGAN results were obtained with an architecture and hyperparameters which
exactly match those proposed by \citet{Hoang18ICLR} for a similar task. This
means that the MGAN models shown
use many more parameters (approx. 310,000) than the GAN and GANG models (approx.
2,000). 
MGAN requires the number of generators to be chosen upfront as a hyperparameter of 
the method.
We chose this to be equal to the number of mixture components, so that MGAN
could cover all modes with one generator per mode. 
We note that PNM does not require such a hyperparameter to be set, nor does PNM
require the  related ``diversity'' hyperparameter of the MGAN method (called $\beta$ in
the MGAN paper).

Looking at Figure \ref{fig:mgan_samples}, we see that MGAN results do seem
qualitatively quite good, even though there is one missed mode (and thus also
one mode covered by 2 generators) on the randomly located components task (right
column).

\begin{figure}[ht]
\centering
\scalebox{0.96}{
\begin{tabular}{ccc}
\hspace{-7mm}
\subfloat{\includegraphics[trim={0cm 2.6cm 0cm 0cm}, clip, width = 0.35\columnwidth]{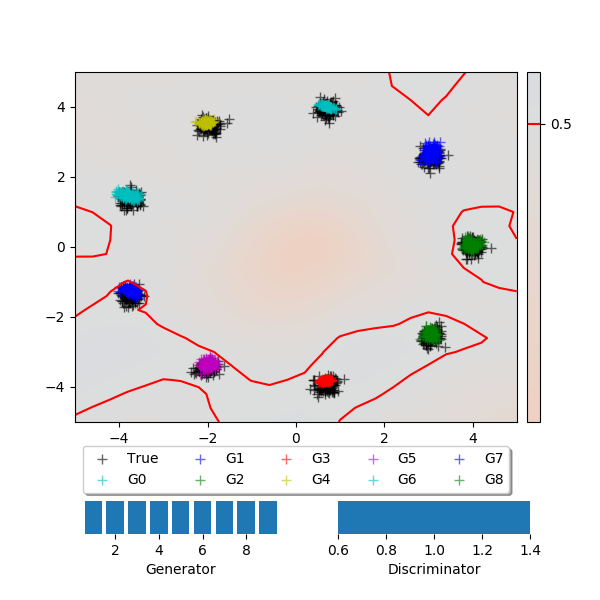}} &
\subfloat{\includegraphics[trim={0cm 2.6cm 0cm 0cm}, clip, width = 0.35\columnwidth]{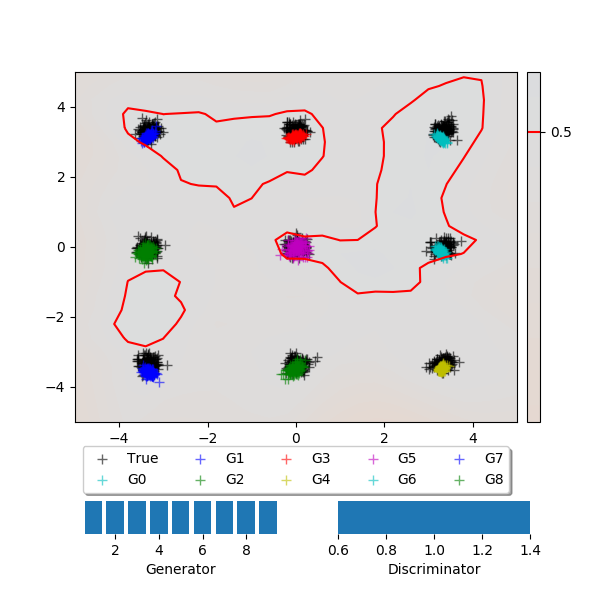}}
\subfloat{\includegraphics[trim={0cm 2.6cm 0cm 0cm}, clip, width = 0.35\columnwidth]{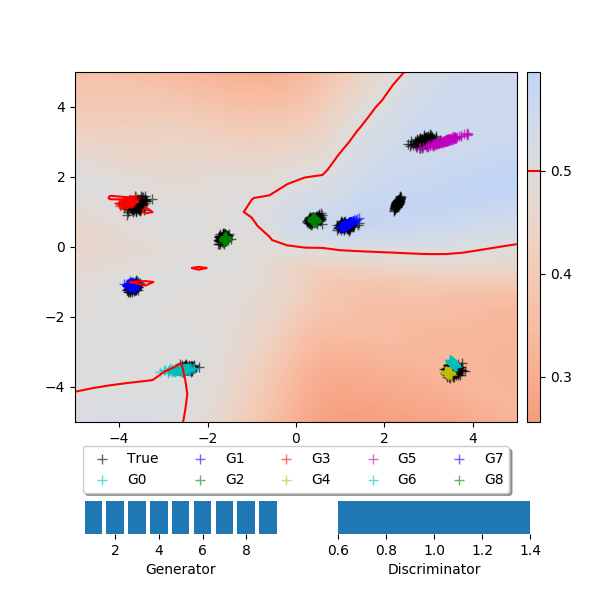}} &
\end{tabular}
}
\caption{Results for MGAN on several mixture of Gaussian tasks with 9 modes. Markers correspond to samples created by each generator.}
\label{fig:mgan_samples}
\end{figure}

\begin{figure*}[h]
\centering
\begin{tabular}{ccc}

\hspace{-0.5cm}
\subfloat{\includegraphics[scale=0.353]{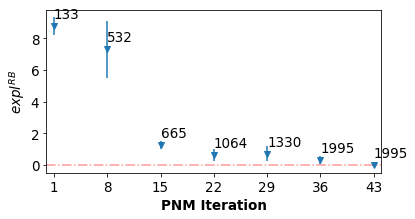} } &
\hspace{-0.5cm}
\subfloat{\includegraphics[scale=0.353]{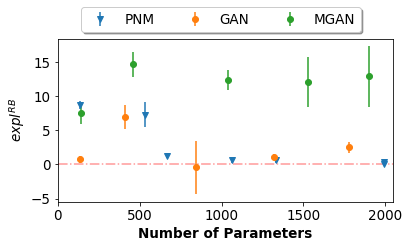} } &
\hspace{-0.5cm}
\subfloat{\includegraphics[scale=0.353]{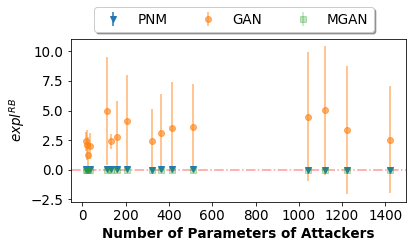} } \\

\hspace{-0.5cm}
\subfloat{\includegraphics[scale=0.353]{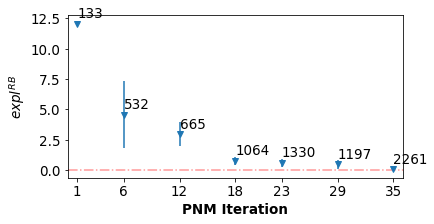} } &
\hspace{-0.5cm}
\subfloat{\includegraphics[scale=0.353]{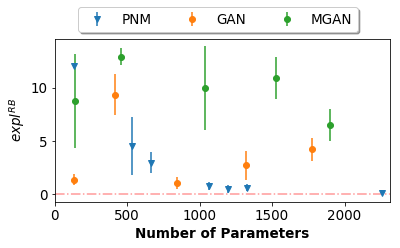} } &
\hspace{-0.5cm}
\subfloat{\includegraphics[scale=0.353]{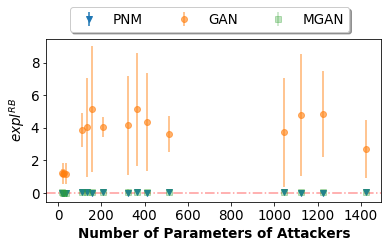} } \\

\hspace{-0.5cm}
\subfloat{\includegraphics[scale=0.353]{figs-exploitability-with-MGAN/random1.png} } &
\hspace{-0.5cm}
\subfloat{\includegraphics[scale=0.353]{figs-exploitability-with-MGAN/random2.png} } &
\hspace{-0.5cm}
\subfloat{\includegraphics[scale=0.353]{figs-exploitability-with-MGAN/random3.png} }

\end{tabular}

\caption{Exploitability results all 9 mode tasks. Top to bottom: round, grid, random.}
\label{fig:exploitabilitySupp}
\end{figure*}

Figure~\ref{fig:exploitabilitySupp} shows our exploitability results for all
three tasks with nine modes. We observe roughly the same trend across the three
tasks. The left column plots show the exploitability  of GANG after different
numbers of iterations (with respect to an attacker of fixed complexity: 453
parameters for the attacking G and C together), as well as the number of
parameters used in the solutions found in those iterations (a sum over all the
networks in the support of the mixture for both players). Error bars indicate
standard deviations over 15 trials. Note how the exploitability of the found
solution monotonically decreases as the PNM iterations increase. 

The middle column shows how exploitable GAN, MGAN and PNM-GANG models of
different complexities are: the x-axis indicates the total number of parameters, while
the y-axis shows the exploitability. The PNM results are the same points also
shown in the left column, but repositioned at the appropriate place on the
x-axis. All data points are exploitability of models that were trained until
convergence. 
Note that here the x-axis shows the complexity in terms of total parameters. The
figure shows an approximately monotonic decrease in exploitability for GANGs,
while GANs and MGANs with higher complexity are still very exploitable in many
cases. In contrast to GANGs, more complex architectures for GANs or MGANs are thus 
not necessarily a way to guarantee a better solution.

Additionally, we investigate the exploitability of the trained models presented
in Figure 1 when attacked by neural networks of varying complexity.  These
results are shown in the right column of Figure \ref{fig:exploitabilitySupp}.
Clearly shown is that the PNM-GANG is robust with near-zero exploitability even
when attacked with high-complexity attackers.  The MGAN models also have low
exploitability, but recall that these models are \emph{much} more complex (GANG
and GAN models have approximately 2,000 parameters, while the MGAN model
involves approximately 310,000 parameters).  Even with such a complex model, in
the `random' task, the MGAN solution has a non-zero level of exploitability,
roughly constant for several attacker complexities.  This is related to the
missed mode and the fact that two of the MGAN generators collapsed to the same
lower-right mode in Figure 1.  In stark contrast to both PNM-GANGs and MGAN, we
see that the converged GAN solution is exploitable already for low-complexity
attackers, again suggesting that the GAN was stuck in an Local Nash Equilibrium
far away from a (global) Nash Equilibrium.

As stated in the main paper, the variance of the exploitability depends
critically on the solution that is attacked. The top row of Figure~\ref{fig:exploit_variance} shows the results of three different attacks of
the $G$ against the $\mA C$ found by the GAN (top), and PNM-GANG (bottom).
We see that in the top row, due to the shape of  $\mA{C}$, the attacking generators 
reliably find the blue region with large classifier scores. 
On the other hand, in the bottom row, we see that the attacking generators
sometimes succeed in finding a blue area, but sometimes get stuck in a local
optimum (bottom left plot).

\begin{figure*}[tb]
\centering
\begin{tabular}{ccc}
\subfloat{\includegraphics[trim={1cm 2.8cm 0.5cm 1cm},clip, scale = 0.3512897]{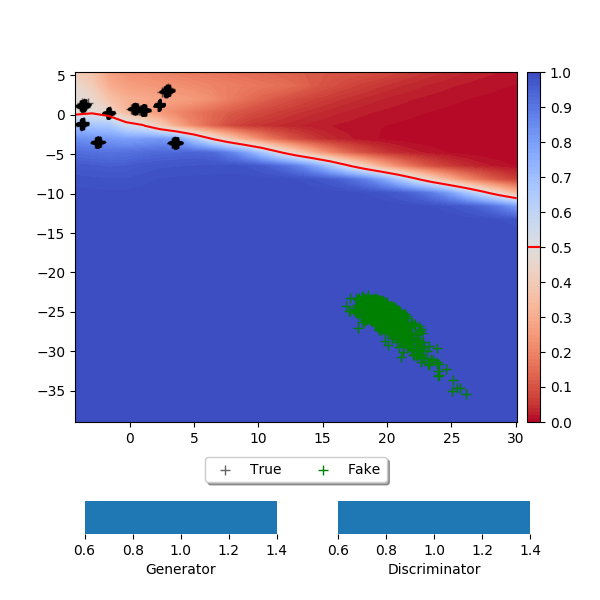}} &
\subfloat{\includegraphics[trim={1cm 2.8cm 0.5cm 1cm},clip, scale = 0.3512897]{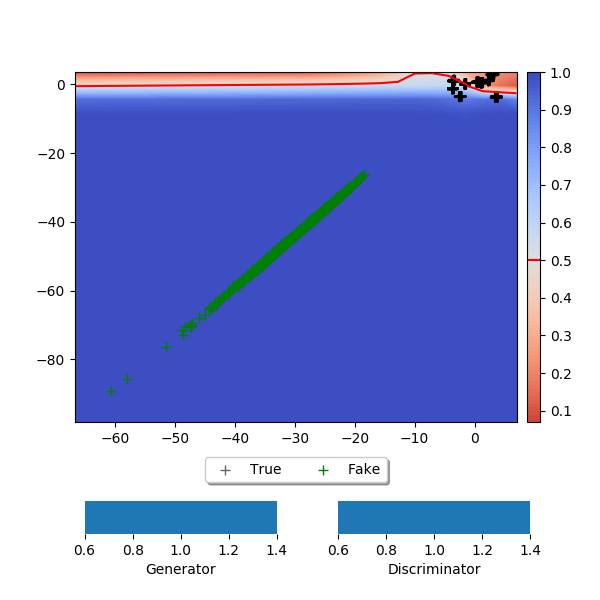}} &
\subfloat{\includegraphics[trim={1cm 2.8cm 0.5cm 1cm},clip, scale = 0.3512897]{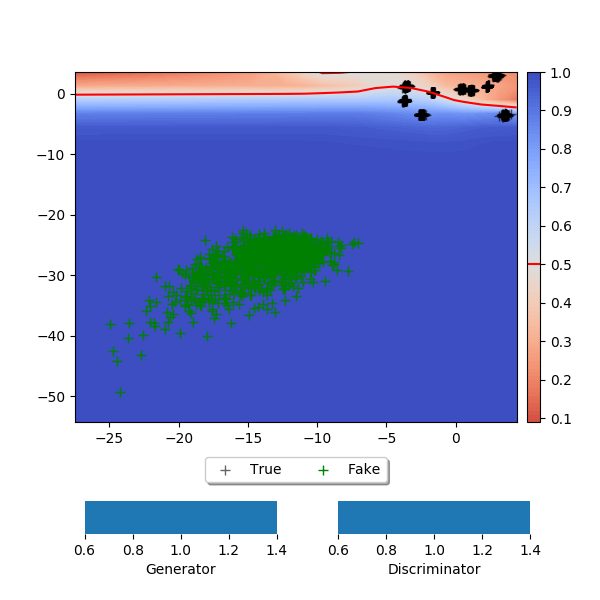}} \\

\subfloat{\includegraphics[trim={1cm 2.8cm 0.5cm 1cm},clip, scale = 0.3512897]{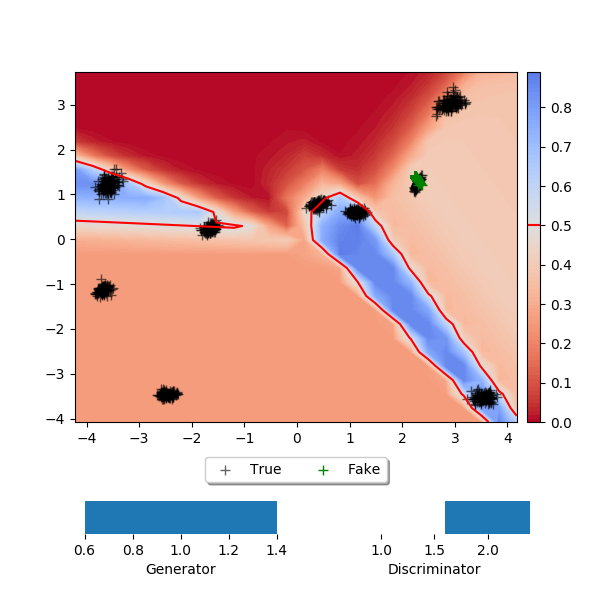}} &
\subfloat{\includegraphics[trim={1cm 2.8cm 0.5cm 1cm},clip, scale = 0.3512897]{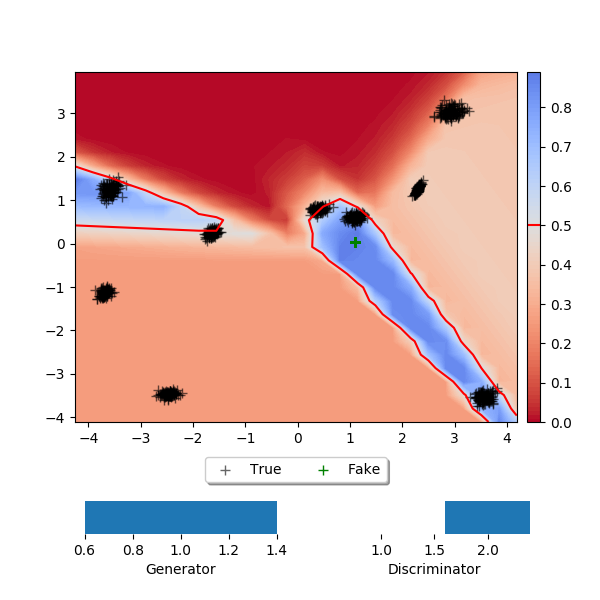}} &
\subfloat{\includegraphics[trim={1cm 2.8cm 0.5cm 1cm},clip, scale = 0.3512897]{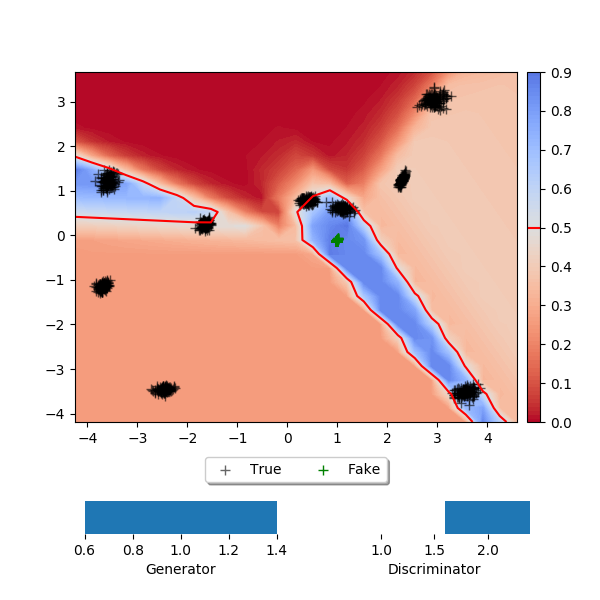}}\\

\end{tabular}
\caption{Found best response generators against a fixed GAN (top) and PNM-GANG (bottom) classifier.}
\label{fig:exploit_variance}
\end{figure*}

These results demonstrate that PNM-GANGs can provide more robust solutions than
GANs/MGANs with the same number of parameters, suggesting that they are closer to a
Nash equilibrium and provide better generative models.

\section{Interleaved training for faster convergence}
\label{app:interleaved}

In order to speed up convergence of PNM-GANG, it is possible to train best responses
of $G$ and $C$ in parallel, giving $G$ access to the intermediate results of
$f_C^{RBBR}$.  The resulting algorithm is shown in Algorithm~\ref{alg:AccelPNM}. 

In this case, the loss with which $G$ is trained depends not only on the scores
given by the current mixture of classifiers, $\mA{C}$, but also on the
classification scores given by $C$. Formally, the new proposed loss for $G$ is
a weighted sum of the losses if it were to be evaluated against $\mA{C}$ and
$C$, respectively. Clearly, the best response computation presented in the
paper corresponds to the case in which the weight of the loss coming from $C$
is zero. Intuitively, this gives the generator player the chance to be
one-step ahead of the discriminator in discovering modes that are not being
currently covered by $\mA{G}$.

\begin{algorithm}
\protect\caption{\funcName{Interleaved Training PNM for GANGs}}
\providecommand{\commentSymb}{//}
\begin{algorithmic}[1]
\small
\State{$\langle \aA{G}, \aA{C} \rangle \gets \funcName{InitialStrategies}()            $}
\State{$\langle \mA{G}, \mA{C} \rangle \gets \langle \{\aA{G}\}, \{\aA{C}\} \rangle    $} \Comment{set initial mixtures}
\While{True}
    \While{Training}
        \State{$ \aA{C} \gets \funcName{GradientStepC} ( \mA{G} ) $}
        \State{$ \aA{G} \gets \funcName{GradientStepG} ( \mA{C}, \aA{C} )  $}  
    \EndWhile
    \State \commentSymb{ Expected payoffs of these `tests' against mixture:}
    \State{$ u_{BRs} \gets \utA{G}(\aA{G},  \mA{C} ) + \utA{C}(\mA{G}, \aA{C}) $}  
    \If{$ u_{BRs} \leq 0 $}
        \State \textbf{break}
    \EndIf
        \State{$SG \gets \funcName{AugmentGame}(SG, \aA{G}, \aA{C})     $}
        \State{$\langle \mA{G}, \mA{C} \rangle \gets \funcName{SolveGame}(SG)          $}
\EndWhile
\State{\textbf{return} $\langle \mA{G}, \mA{C} \rangle $} 
\end{algorithmic}
\label{alg:AccelPNM}
\end{algorithm}

This technique may interfere with the convergence guarantees provided above.
It might not be the case that we are in fact computing a resource-bounded
best-response for $G$ anymore. However, in practice it performs very well: it 
reduces the number of needed PNM iterations, but does not affect the quality of the found solutions, as demonstrated in 
\fig\ref{fig:interleaved}.

\begin{figure*}[p]
\centering
\begin{tabular}{cc}
%
\textbf{Normal best responses} & \textbf{Interleaved training} \\
%
& \\
29 PNM iterations: & 24 PNM iterations:\\
\subfloat{\includegraphics[trim={0cm 0cm 0cm 1.85cm},clip, scale = 0.4]{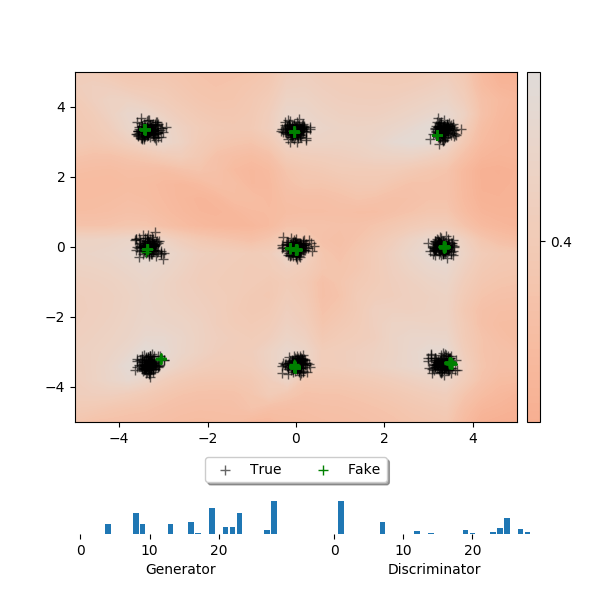}} &
\subfloat{\includegraphics[trim={0cm 0cm 0cm 1.85cm},clip, scale = 0.4]{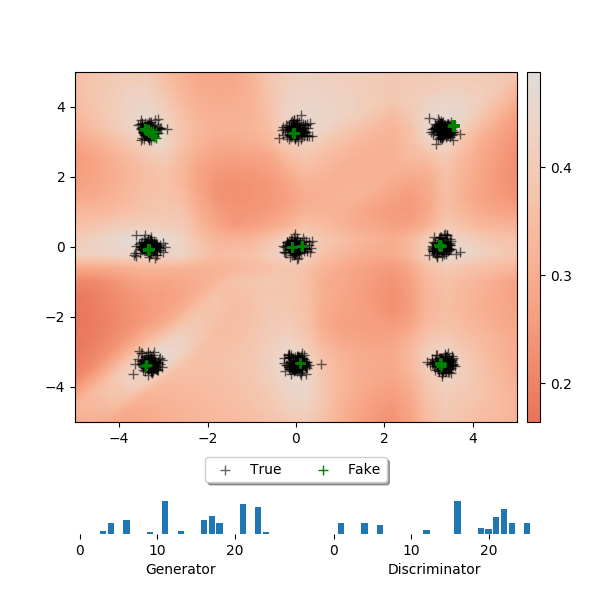}} \\
%
%
& \\
40 PNM iterations: & 25 PNM iterations:\\
\subfloat{\includegraphics[trim={0cm 0cm 0cm 1.85cm},clip, scale = 0.4]{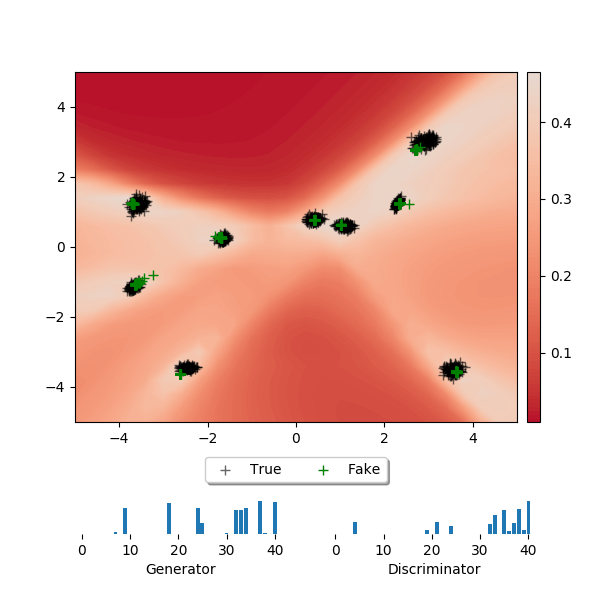}} &
\subfloat{\includegraphics[trim={0cm 0cm 0cm 1.85cm},clip, scale = 0.4]{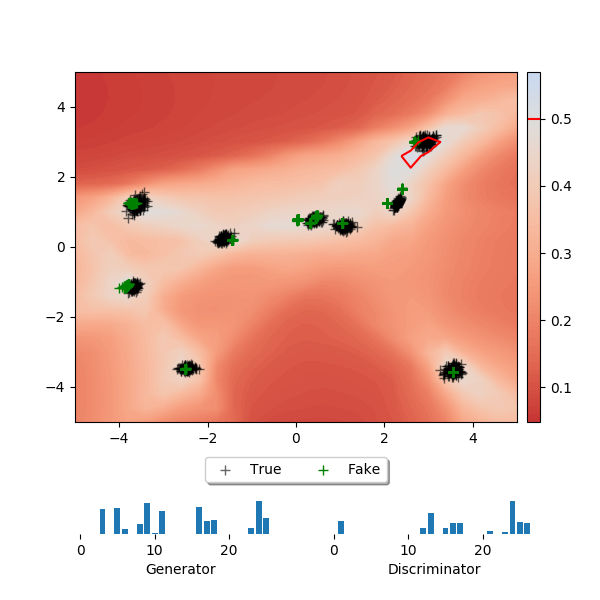}} \\
%
%
& \\
40 PNM iterations: & 18 PNM iterations:\\
\subfloat{\includegraphics[trim={0cm 0cm 0cm 1.85cm},clip, scale = 0.4]{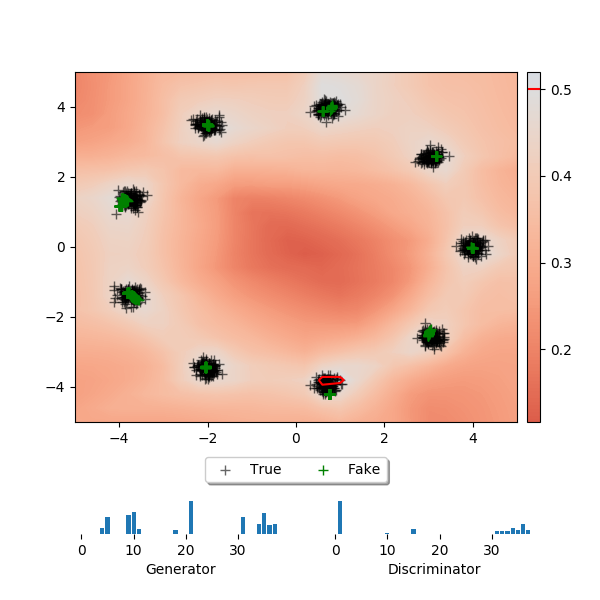}} &
\subfloat{\includegraphics[trim={0cm 0cm 0cm 1.85cm},clip, scale = 0.4]{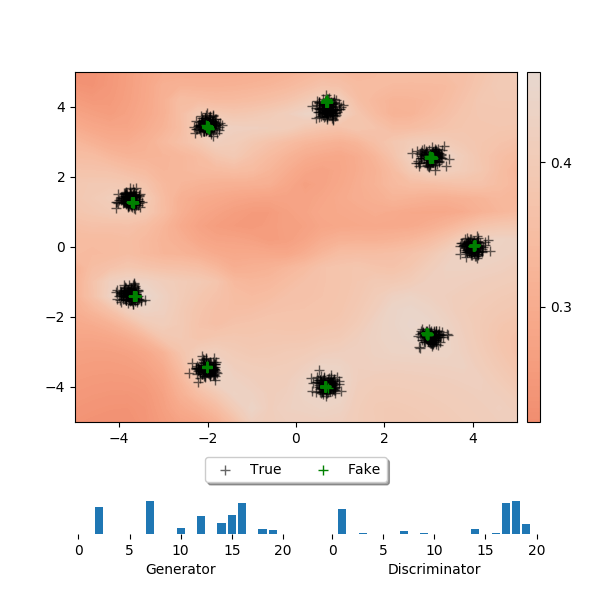}}\\
%
\end{tabular}

\caption{Interleaved training. The left column shows normal training, as used in the main paper; the 
right column shows interleaved training. In all cases, we manually picked the first iteration where
all modes are covered. This consistently took less iterations for interleaved training, while solution
quality was comparable.}
\label{fig:interleaved}
\end{figure*}

\end{document}